\documentclass[letterpaper]{article} \usepackage[preprint]{aaai2027}
\usepackage[hyphens]{url} \usepackage{graphicx}   \usepackage{natbib} \usepackage{caption} 

\usepackage{amsmath,amssymb,amsfonts,amsthm}
\usepackage{mathtools}
\usepackage{algorithm}
\usepackage{algorithmic}
\usepackage{subcaption}
\usepackage{booktabs}
\usepackage{multirow}
\usepackage{xcolor}
\usepackage{soul}
\definecolor{grpCloth}{HTML}{2E5480}
\definecolor{grpFoot}{HTML}{A33A3C}
\definecolor{grpAcc}{HTML}{6B6A6B}
\theoremstyle{plain}
\newtheorem{theorem}{Theorem}[section]

\newtheorem{lemma}[theorem]{Lemma}
\newtheorem{corollary}[theorem]{Corollary}
\theoremstyle{definition}
\newtheorem{definition}[theorem]{Definition}
\newtheorem{assumption}[theorem]{Assumption}
\theoremstyle{remark}
\newtheorem{remark}[theorem]{Remark}

\newcommand*{\ie}{\emph{i.e.}{}}
    \newcommand{\E}{\mathbb{E}}
\newcommand{\R}{\mathbb{R}}
\renewcommand{\P}{\mathbb{P}}

\DeclareMathOperator*{\argmax}{argmax}
\DeclareMathOperator*{\argmin}{argmin}

\DeclareMathOperator{\dist}{dist}
\DeclareMathOperator{\Dir}{Dir}         
\DeclareMathOperator{\CVaR}{CVaR}
\DeclareMathOperator{\Var}{Var}

\newcommand{\Alpha}{\boldsymbol{\alpha}}

\newcommand{\baseline}{baseline }
\newcommand{\target}{target }

\title{Deciding When to Decide:\\Testing Operational Suboptimality Under Distributional Shift}
\author{
    Minxing Zheng, Holly Wiberg, Shixiang Zhu
}
\affiliations{
    Carnegie Mellon University
}

\begin{document}

\maketitle

\begin{abstract}
Deployed decisions are often optimized once and retained because updates impose operational, regulatory, or switching costs. As operating conditions change, when should such decisions be re-optimized? We study this question for stochastic optimization when the objective's functional form is known but the decision maker's trade-offs are encoded by an unknown preference parameter. Standard distribution-shift tests are poorly aligned with this goal: they can flag detectable yet decision-irrelevant changes without determining whether the incumbent decision has become materially suboptimal. We propose \texttt{RADAR} (Regret-based Assessment of Decision Adequacy and Risk), a decision-focused framework that uses inverse optimization to infer latent preferences and tests the deployed decision's optimality gap under the current distribution. By targeting regret, \texttt{RADAR} ignores decision-irrelevant shifts while detecting changes that warrant re-optimization. We develop two-sample and sequential changepoint procedures and establish asymptotic guarantees for Type-I error and power. Across synthetic optimization problems, a semi-synthetic capacity allocation task, and police-zone planning, \texttt{RADAR} more reliably distinguishes harmful from harmless shifts than decision-agnostic alternatives.
\end{abstract}

\section{Introduction}
Data-driven optimization increasingly supports consequential decisions in operational and public-sector domains, such as resource allocation \cite{bertsimas2020predictive,chen2022data,bertsimas_predictions_2021}, districting and zone design \cite{camacho2015multi,liu2016pear,zhu2022data,xing2025black}, and routing \cite{chu2023data,lai2022data,Toth-2002}. Often, a decision is optimized using historical context data and then implemented over an extended horizon, during which the environment may drift away from the design conditions and render the decision suboptimal. Yet re-optimizing can be costly, risky, and slow: a revised decision may require model validation, regulatory review, operational coordination, or changes to downstream workflows \cite{Basso-2019,Brailsford2005BarriersORHealthcare}. Deployed decisions therefore exhibit ``decision inertia'', remaining in use well after the data distribution has shifted \cite{akaishi2014autonomous,Ferrer-2016}. This creates a central deployment challenge: \textit{deciding when a deployed decision remains adequate and when environmental shift warrants re-optimization}. Addressing it involves two difficulties.

\begin{figure}[t]
    \centering
    \begin{minipage}[t]{0.31\linewidth}
        \centering
        \includegraphics[width=\linewidth]{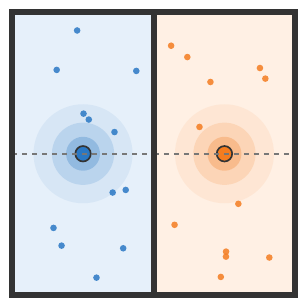}
        
\end{minipage}
    \hfill
    \begin{minipage}[t]{0.31\linewidth}
        \centering
        \includegraphics[width=\linewidth]{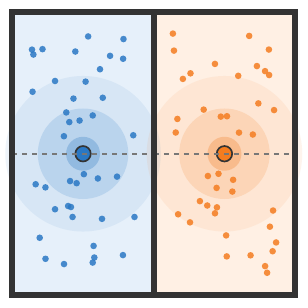}
        
\end{minipage}
    \hfill
    \begin{minipage}[t]{0.31\linewidth}
        \centering
        \includegraphics[width=\linewidth]{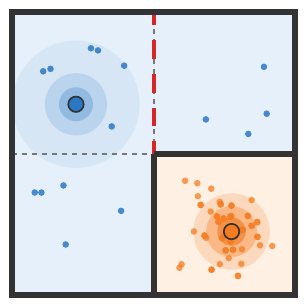}

\end{minipage}
\includegraphics[width=\linewidth]{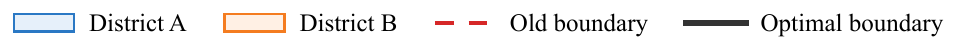}
    \vspace{-0.2in}
    \caption{Stylized two-zone redistricting, where dots denote demand locations and boundaries are chosen to balance demand. Left: The baseline-optimal boundary. Middle: Uniform demand growth shifts the distribution but leaves the boundary optimal. Right: Localized demand relocation makes the old boundary (red dashed) suboptimal and favors an L-shaped boundary (black solid).}
    \label{fig:motivation_demo_districting}
\end{figure}

First, not all context distribution shifts warrant re-optimization. A natural approach to addressing distribution shift is to apply a two-sample test or a changepoint detector to the context distribution. However, this is misaligned with the operational goal. Some shifts are statistically detectable yet decision-irrelevant, whereas small context shifts may induce large suboptimality for the deployed decision. Thus, the relevant quantity is not the distributional discrepancy itself, but the optimality gap of the deployed decision under the current operating conditions. Second, the objective function that generated the deployed decision is often only partially observed. If the full objective were known, assessing decision adequacy would reduce to estimating the target-domain optimality gap of the deployed decision and testing whether it exceeds a tolerance. We consider the more realistic setting where the deployed decision arises from an underlying stochastic optimization problem whose functional form is known, yet the parameters that encode the decision maker's preferences and trade-offs are unobserved. One can observe the deployed decision and context data from historical and current operating environments, but the objective parameters must be inferred. This partial observation setting is common in practice: even when the broad decision model is known, trade-offs among cost, reliability, risk, service quality, and fairness may be implicit, proprietary, or shaped by legacy workflows and human judgment. Thus, assessing decision adequacy requires inferring the latent preference structure that rationalizes the deployed decision and then evaluating whether it remains suitable under the new environment.

Figure~\ref{fig:motivation_demo_districting} illustrates both challenges through a stylized police districting example. A planner locates police stations and partitions the service area while balancing response time against officer workload. Under the baseline demand pattern (left), the resulting plan uses a vertical district boundary; although the deployed stations and boundary are observable, the planner's weighting of the two objectives is not. Uniform demand growth (middle) changes the operating environment but preserves the spatial balance of demand, leaving the original boundary optimal. By contrast, localized demand relocation (right) makes the original boundary suboptimal and warrants a different partition. The figure therefore highlights why detecting demand shift alone is insufficient: assessing whether redistricting is warranted requires evaluating the operational consequences of the shift under the latent trade-off that rationalized the original plan.

We propose \texttt{RADAR} (Regret-based Assessment of Decision Adequacy and Risk), a decision-focused framework that combines inverse optimization with regret-based inference to assess whether a deployed decision remains adequate under distribution shift.  We summarize our contributions as follows. First, we formulate monitoring under distribution shift as testing whether a deployed decision remains approximately optimal, in both two-sample and sequential settings. Second, we develop \texttt{RADAR}, which recovers the latent preference parameters that rationalize the deployed decision via inverse optimization and evaluates the resulting optimality gap through a sample-split regret test, requiring only that the recovered preferences be identified well enough to determine adequacy. Third, we establish asymptotic guarantees for Type-I error and power, and show on synthetic and real-world problems that \texttt{RADAR} detects decision-relevant shifts while ignoring changes that do not warrant re-optimization.

\vspace{.1in}
\noindent\emph{\ul{Related Work}}.
Distribution shift is widely studied in statistical learning, including dataset \cite{gama2014survey,quinonero2008dataset}, covariate \cite{bickel2009discriminative,shimodaira2000improving}, selection-induced \cite{cortes2008sample,huang2006correcting} and label \cite{garg2020unified,guo2020ltf} shift, with diagnostics for deployed machine-learning systems \cite{guan2025keeping,kulinski2023towards}. Classical tools detect and evaluate such changes: two-sample tests \cite{anderson1962distribution,anderson1952asymptotic,smirnov1948table} for static comparisons and changepoint detection \cite{basseville_detection_1993,harchaoui2008kernel,Killick01122012,Matteson02012014,page1954continuous} for sequential monitoring, together with learning-based variants such as classifier two-sample tests \cite{lopez2017revisiting,kim2021classification} and learned changepoint detectors \cite{hushchyn2021generalization,li2024automatic}.  Closest to our setting, \citet{podkopaev2022tracking} monitor the risk of a deployed model by testing whether its target-domain risk exceeds its source-domain risk. However,  an increase in risk need not imply that the deployed model or decision has become inadequate: risk may increase while the optimizer remains unchanged, or remain stable while a better solution exists. We instead test the incumbent decision's \emph{optimality gap} under the target distribution, directly targeting whether re-optimization is warranted.

Our framework also connects to inverse optimization (IO) \cite{Ahuja-2001,chan2025inverse,IYENGAR2005319,Birge-2017}, which infers objective or constraint parameters \cite{CHAN2020415,GHOBADI2021829} that make observed decisions optimal or approximately optimal for a proposed forward optimization model. Classical IO typically enforces exact inverse feasibility, requiring the observed decisions to be optimal under the recovered parameters \cite{heuberger2004inverse}. Data-driven IO relaxes this requirement by fitting parameters through loss functions that measure deviations from optimality, allowing for noisy observations, imperfect information, or model misspecification \cite{Aswani-2018,bertsimas2015data,mohajerin2018data}. Beyond parameter recovery, recent work has leveraged inverse-optimization geometry for decision risk assessment, using conformal inference to provide distribution-free certificates on the suboptimality risk of candidate decisions \cite{zhou2026conformalized}. Recent work also studies IO in stochastic settings \cite{NEURIPS2022_e3c34738}, as well as the recovery of risk preferences in risk-minimization problems \cite{li2021inverse}.

Predict-then-optimize \cite{elmachtoub_2021_spo, NEURIPS2021_b943325c} and decision-focused learning \cite{Mandi-2022,mandi2025feasibility,mandi2024decision,wang2025gen,wilder_2019_dfl} methods align prediction with downstream optimization quality by training predictive models using decision-aware objectives, such as decision loss, regret, or task-specific optimization loss \cite{Schutte-2024,NEURIPS2022_0904c7ed}. Related work develops differentiable optimization layers and gradient estimators for propagating learning signals through downstream optimization problems \cite{agrawal2019differentiable,Brandon-2017,poganvcic2019differentiation}. These methods ask how to learn predictions that induce high-quality decisions. Our work asks a complementary question: given a deployed decision and a known optimization structure, has the decision's optimality changed enough under distribution shift to warrant re-optimization? Thus, we bring a decision-focused criterion to hypothesis testing and deployment monitoring.

\section{Problem Setup}

We consider two domains: a \emph{\baseline domain} and a \emph{\target domain}. The baseline domain represents the historical operating environment in which the decision was optimized, while the target domain represents the current operating environment. In the \baseline domain, we observe i.i.d. contexts $X_1, \dots, X_n \sim \P_0,$ and a deployed decision $z_0 \in \mathcal Z$, which is assumed to be optimal for a stochastic optimization problem under the baseline context distribution, \ie,
\begin{equation}
\label{eq:forward_problem}
z_0 \in \argmin_{z\in\mathcal Z} \E_{\P_0}\!\left[f(z;X,\theta^\star)\right].
\end{equation}
Here, $f(z;X,\theta)$ is a stochastic objective with known functional form but unknown preference parameter vector $\theta^\star\in\Theta\subseteq\mathbb{R}^d$. The vector $\theta^\star$ encodes the decision maker's latent preferences, such as implicit trade-offs among cost, reliability, fairness, and risk. In the \target domain, we observe new i.i.d. contexts $X_{n+1}, \dots, X_{n+m} \sim  \P_1,$ where $\P_1$ is the \target domain context distribution and may differ from $\P_0$. The deployed decision $z_0$ remains fixed and feasible. Our objective is to assess whether $z_0$ remains approximately optimal under $\P_1$. To this end, we define the \emph{risk-optimality gap}
\begin{equation}
\label{eq:gap_def}
\Delta(z;\theta,\P)
=
\E_\P\left[f(z; X, \theta)\right]
-
\min_{z' \in \mathcal Z}
\E_\P\left[f(z'; X, \theta)\right],
\end{equation}
which measures the excess expected objective value, \ie, risk incurred by decision $z$ relative to the optimal decision under distribution $\mathbb P$ and parameter $\theta$. We are interested in the \target domain gap
\begin{equation}
\label{eq:oracle_optimality_gap}
   \Delta^\star \coloneqq \Delta(z_0;\theta^\star,\mathbb P_1), 
\end{equation}
evaluated under the true preference parameter $\theta^\star$ and the \target domain distribution $\mathbb P_1$. Given a tolerance level $\tau\ge 0$, we test whether this gap exceeds the acceptable level:
\begin{equation}
\label{eq:hypothesis_test}
H_0:~\Delta(z_0; \theta^\star, \mathbb P_1) \le \tau 
\;\;\text{vs.}\;\;
H_1:~\Delta(z_0; \theta^\star, \mathbb P_1) > \tau.
\end{equation}
Rejecting $H_0$ indicates that the deployed decision incurs excess risk beyond the tolerance $\tau$, suggesting that it should be re-optimized for the \target domain. The tolerance $\tau$ may be specified either on an absolute risk scale or on a relative scale. In particular, $\tau=0$ corresponds to exact optimality, whereas $\tau>0$ permits a nonzero excess-risk tolerance. A relative criterion may also be defined by normalizing the gap by the optimal risk under $(\theta^\star,\mathbb P_1)$.

\begin{remark}[Extension beyond expectation risk]
We present the main setup under expectation risk for clarity, but the framework extends to more general risk functionals $\rho$ by replacing $\mathbb E_{\mathbb P}[f(z;X,\theta)]$ with $\rho^{\mathbb P}[f(z;X,\theta)]$. In our experiments, we include a conditional value-at-risk (CVaR) example, which captures tail risk, to illustrate this extension beyond expectation-based risk. More details can be found in Appendix~\ref{appendix:theory_proof}.
\end{remark}

\section{Proposed Method}
\label{sec:hypo-test}

We propose \texttt{RADAR} (Regret-based Assessment of Decision Adequacy and Risk), a decision-aware testing framework for assessing whether a deployed decision remains approximately optimal under distributional shift. Our goal is not to detect generic context-distribution shifts, but to test whether the deployed decision $z_0$ satisfies the optimality tolerance in \eqref{eq:hypothesis_test} under the \target domain. Although \eqref{eq:hypothesis_test} involves two domains, it is not a classical two-sample testing problem, because a context shift is relevant only when it induces a nontrivial optimality gap of the deployed decision. Hence, directly comparing $\mathbb P_0$ and $\mathbb P_1$ may be overly conservative or irrelevant for assessing decision adequacy.

Our approach resolves this challenge by directly targeting the optimality gap in \eqref{eq:oracle_optimality_gap}. Specifically, our framework combines inverse optimization on \baseline domain data to infer a latent preference parameter that rationalizes the observed \baseline decision $z_0$ with a data-driven approximation to the \target domain optimizer. These components yield an estimator of the target optimality gap. We then use a sample-split evaluation step to obtain valid inference for this estimated gap. This leads to a practical three-step procedure for assessing decision adequacy, which we further extend to a sequential setting for decision-relevant change-point detection.

\subsection{Decision-Focused Risk-Optimality Gap Test}
We now describe the core procedure of \texttt{RADAR}. The central challenge is that the hypothesis in \eqref{eq:hypothesis_test} depends on two unknown objects: the latent preference parameter $\theta^\star$ that rationalizes the observed decision $z_0$, and the \target domain optimizer needed to evaluate the oracle optimality gap in \eqref{eq:oracle_optimality_gap}. Our approach addresses this difficulty through a three-step, data-driven procedure tailored to the decision-adequacy testing objective.

\begin{figure}[!t]
    \centering
    \includegraphics[width=\linewidth]{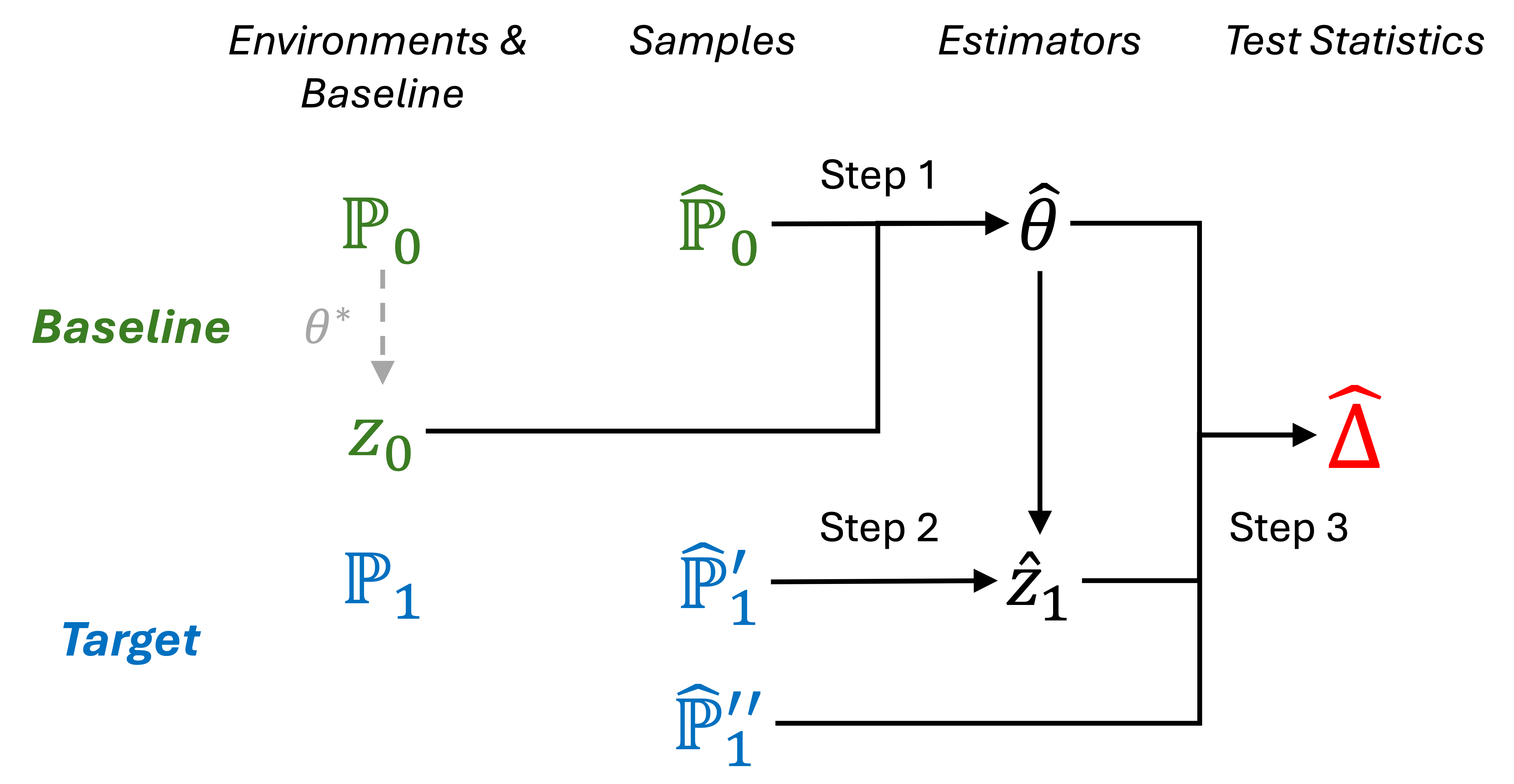}
    \caption{Overview of the proposed \texttt{RADAR} framework.}
    \label{fig:workflow_demo}
    \vspace{-0.2em}
\end{figure}
In Step~$1$, we use inverse optimization on \baseline domain data to estimate the latent preference parameter $\theta^\star$, yielding a plug-in estimator $\hat\theta$. In Step~$2$, we combine $\hat\theta$ with \target domain data to construct a data-driven approximation of the \target domain optimizer. Together, these two steps provide an empirical estimate of the deployment optimality gap. In Step~$3$, we use a sample-split evaluation scheme to obtain a valid inference for this estimated gap. Figure~\ref{fig:workflow_demo} summarizes the workflow. We next describe these three steps in detail.

\vspace{.1in}
\noindent\emph{\ul{Step 1: Preference Parameter Estimation}}.
Recall that under the forward model \eqref{eq:forward_problem}, the baseline decision $z_0$ is optimal for the baseline distribution  $\mathbb P_0$ under $\theta^\star$. We estimate $\theta^\star$ from the baseline-environment data $\{X_i\}_{i=1}^{n}$ via inverse optimization. Specifically, letting $\widehat{\mathbb P}_0$ denote the empirical distribution of the baseline sample, we compute
\begin{equation}
\label{eq:iop}
\hat\theta \in \argmin_{\theta\in\Theta}
\Big\{\E_{\widehat{\mathbb P}_0}\big[f(z_0;X,\theta)\big] - \inf_{z\in\mathcal Z} \E_{\widehat{\mathbb P}_0}\big[f(z;X,\theta)\big]
\Big\}.
\end{equation}
This corresponds to minimizing the empirical optimality gap of $z_0$ under candidate parameters $\theta$. Throughout, we require only that $\hat\theta$ be \emph{risk-consistent} for $\theta^\star$; our inferential guarantees do not depend on the particular inverse-optimization algorithm beyond this property.

\vspace{.1in}
\noindent\emph{\ul{Step 2: Constructing a Deployment Benchmark}}.
To assess whether the deployed decision remains adequate under the \target domain, the test must compare it with the decision that would be optimal under the \target distribution, as in \eqref{eq:gap_def}. This oracle optimizer, $z_1^\star \in \argmin_{z\in\mathcal Z} \E_{{\mathbb P}_{1}}\!\left[f(z;X,\theta^\star)\right],$ is not observed and cannot be computed directly because both the latent preference parameter $\theta^\star$ and the target distribution $\mathbb P_1$ are unknown. We therefore construct a data-driven deployment benchmark using the plug-in estimate $\hat\theta$ and target-domain samples.

Specifically, we randomly partition the $m$ target-domain samples into two disjoint subsamples of sizes $m'$ and $m''$, with $m' + m'' = m$. Let $\widehat{\mathbb P}_{1}'$ and $\widehat{\mathbb P}_{1}''$ denote the corresponding empirical distributions. In our experiments, we use a balanced split, $m'=\lfloor m/2\rfloor$ and $m''=m-m'$. This choice is also suggested by the convergence rate in Theorem~\ref{thm:delta_hat_rate}: for a fixed target-domain sample size $m$, a balanced allocation minimizes the contribution $(m')^{-1/2}+(m'')^{-1/2}$ in the stated rate bound. Using the first subsample, we compute 
\begin{equation}
\label{eq:z_1_hat_method}
\hat z_1 \in \argmin_{z\in\mathcal Z} 
\E_{\widehat{\mathbb P}_{1}'}\!\left[f(z;X,\hat\theta)\right].
\end{equation}
We use $\hat z_1$ as a data-driven benchmark for the unobserved target-domain optimizer $z_1^\star$. The second subsample, associated with $\widehat{\mathbb P}_{1}''$, is reserved solely for evaluation. Since $\hat z_1$ is constructed only based on $\widehat{\mathbb P}_{1}'$, it is statistically independent of the evaluation data. This separation between benchmark construction and evaluation is necessary for valid inference on the deployment optimality gap.

\vspace{.1in}
\noindent\emph{\ul{Step 3: Testing Procedure}}.
Given the estimator $\hat\theta$ in \eqref{eq:iop} and the estimated deployment benchmark $\hat z_1$ \eqref{eq:z_1_hat_method}, the deployment gap $\Delta(z_0;\theta^\star,\mathbb{P}_1)$ admits the following empirical analog
\begin{equation}
\label{eq:gap_hat_step3}
\widehat\Delta
=
\E_{\widehat{\mathbb P}_1''}\big[f(z_0;X,\hat\theta)\big]
-
\E_{\widehat{\mathbb P}_1''}\big[f(\hat z_1;X,\hat\theta)\big],
\end{equation}
where $\widehat{\mathbb P}_1''$ denotes the empirical distribution of the evaluation subsample. The statistic $\widehat\Delta$ is a sample-average functional computed on $m''$ i.i.d. observations. Conditional on $(\hat\theta,\hat z_1)$, sample splitting prevents evaluation-data reuse; finite-sample preference and benchmark errors are handled separately in the analysis below.

For testing, we use a one-sided Wald test based on held-out target-domain loss differences, which rejects when
\begin{equation}
\label{eq:T_wald}
    T_{\rm Wald}
=
\frac{\sqrt{m''}(\widehat\Delta-\tau)}{\hat\sigma}
>
z_{1-\alpha},
\end{equation}
where $\hat\sigma$ is the sample standard deviation of the held-out loss differences $f(z_0;X,\hat\theta)-f(\hat z_1;X,\hat\theta)$ and $\hat\sigma/\sqrt{m''}$ is the standard error of $\widehat\Delta$. $z_{1-\alpha}$ is the $(1-\alpha)$-quantile of the standard normal distribution. Under expectation-based risk, this follows from the central limit theorem applied to the evaluation split, conditional on $(\hat\theta,\hat z_1)$. More generally, for other risk functionals, especially nonsmooth ones, or in small-sample settings, one may use a bootstrap test; we defer this extension to Appendix~\ref{app:test_details}.

\vspace{.1in}
\noindent\emph{\ul{Sequential Decision-Adequacy Monitoring Extension}}.

We further extend \texttt{RADAR} to sequential monitoring of decision adequacy. We first estimate the preference parameter $\widehat\theta_0$ from a burn-in baseline sample and hold it fixed throughout monitoring. At each monitoring time $t$, we form a window $\mathcal W_t$ of recent observations, randomly partition it into a benchmark-construction split and an evaluation split, and apply the same procedure as in Step~$2$ to estimate the current optimality gap $\widehat\Delta_t$ of the incumbent decision. We raise an alarm at the first monitoring time for which the test indicates that $\widehat\Delta_t$ exceeds the tolerance $\tau$. This construction focuses detection on operationally consequential changes, including both sudden shifts and gradual degradation, while ignoring distributional changes that leave the incumbent decision adequate. Appendix~\ref{subsec:cpd} and Algorithm~\ref{alg:dr_cpd} provide the formal hypotheses and stopping rule.

\subsection{Theoretical Guarantees}
\label{sec:theory}

We summarize the statistical guarantees for \texttt{RADAR}. The goal is to show that the estimated gap $\widehat\Delta$ concentrates around the oracle gap $\Delta^\star$ up to the intrinsic ambiguity caused by the latent preference parameter. The analysis follows the three steps of \texttt{RADAR}: preference estimation, benchmark construction, and independent evaluation. By the definitions in \eqref{eq:gap_def} and \eqref{eq:gap_hat_step3}, a rearrangement yields
\begin{equation}
\label{eq:delta_decomp}
\begin{aligned}
\widehat\Delta-\Delta^\star
&=\underbrace{\Delta(z_0;\hat\theta,\P_1)-\Delta(z_0;\theta^\star,\P_1)}
   _{\text{($i$) preference}}\\
&\quad+\underbrace{\textstyle\min_{z\in\mathcal Z}\E_{\P_1}f(z;X,\hat\theta)
   -\E_{\P_1}f(\hat z_1;X,\hat\theta)}_{\text{($ii$) benchmark}}\\
&\quad+\underbrace{(\widehat\P_1''-\P_1)\big[f(z_0;X,\hat\theta)
   -f(\hat z_1;X,\hat\theta)\big]}_{\text{($iii$) evaluation}}.
\end{aligned}
\end{equation}
Term~\textup{($i$)} captures the error from inferring the latent preference parameter, term~\textup{($ii$)} captures the suboptimality of the data-driven target-domain benchmark, and term~\textup{($iii$)} is the finite-sample evaluation noise on the independent split.

\vspace{.1in}
\noindent\emph{\ul{Regularity Conditions}}.

We impose the following mild regularity conditions that support consistent preference and benchmark estimation, as well as asymptotically valid and powerful hypothesis testing; formal statements and proofs are deferred to Appendix~\ref{appendix:theory_proof}. Specifically, we assume $(i)$ compact decision and preference spaces (Assumption~\ref{ass:compact}); $(ii)$ an objective function that is measurable in $X$ and continuous in $(z,\theta)$ for almost every $X$ (Assumption~\ref{ass:f_regular}); $(iii)$ a square-integrable envelope under both $\mathbb P_0$ and $\mathbb P_1$ (Assumption~\ref{ass:envelope}); and $(iv)$ Lipschitz continuity in $(z,\theta)$ with a square-integrable Lipschitz modulus under both $\mathbb P_0$ and $\mathbb P_1$ (Assumption~\ref{ass:f_lipschitz}). Together, these conditions yield the uniform convergence of the empirical objectives needed to consistently estimate the preference parameter and the deployment benchmark, and hence the gap $\widehat\Delta$ up to the intrinsic inverse-identification ambiguity. The stated convergence rate additionally relies on the local inverse-stability condition in Assumption~\ref{assm:error_bound_P0}, while asymptotic Type-I error control and power against fixed alternatives further require the conditional central limit theorem in Assumption~\ref{ass:eval_clt}.

Since $\theta^\star$ is unobserved, $z_0$ may be rationalized by multiple preference parameters. Let
\[
\Theta_{\rm inv}(z_0;\mathbb P_0)
:=
\left\{\theta\in\Theta:
z_0\in
\argmin_{z\in\mathcal Z}
\E_{\mathbb P_0}\big(f(z;X,\theta)\big)
\right\}.
\]
If $\Theta_{\rm inv}$ is not a singleton, an inverse-optimization estimator can recover a $\theta \in \Theta_{\rm inv}$, not necessarily $\theta^\star$.

\begin{definition}[Target-domain inverse sensitivity]
\label{def:inverse_sensitivity}
Define the identified range of target-domain gaps by
\[
\begin{aligned}
\underline\Delta_{\rm inv}
&:=\inf_{\theta\in\Theta_{\rm inv}}
\Delta(z_0;\theta,\mathbb P_1),\\
\overline\Delta_{\rm inv}
&:=\sup_{\theta\in\Theta_{\rm inv}}
\Delta(z_0;\theta,\mathbb P_1),
\end{aligned}
\]
and its width by $\delta_{\rm inv}:=\overline\Delta_{\rm inv}-\underline\Delta_{\rm inv}$, equivalently the largest disagreement in target gap between any two baseline-rationalizing preferences.
\end{definition} 
The quantity $\delta_{\rm inv}$ is not sampling error but irreducible ambiguity induced by inverse optimization. The adequacy conclusion is determinate whenever $\overline\Delta_{\rm inv}\le\tau$ or $\underline\Delta_{\rm inv}>\tau$, and only the position of the identified range relative to $\tau$ matters, not its width. When the range straddles $\tau$, two observationally equivalent preferences lie on opposite sides of the audit threshold, so no level-$\alpha$ test can have nontrivial power against the harmful one. Resolving such a case requires additional preference-identifying information rather than more target samples, since the latter reduce sampling uncertainty but leave $[\underline\Delta_{\rm inv},\overline\Delta_{\rm inv}]$ unchanged. Absent that, the sensitivity-adjusted threshold below retains validity at the cost of power. Appendix~\ref{app:io_ambiguity} characterizes checkable conditions under which $\delta_{\rm inv}=0$ and discusses its estimation and sensitivity analysis.

We next show that the proposed estimator is consistent for the oracle deployment gap up to inverse-identification ambiguity under the above regularity conditions. The detailed term-by-term bounds for \textup{($i$)}--\textup{($iii$)} are deferred to Appendix~\ref{appendix:theory_proof}.

\begin{theorem}[Consistency and rate of $\widehat\Delta$]
\label{thm:delta_hat_rate}
Suppose Assumptions~\ref{ass:compact} and~\ref{ass:f_regular} hold, and the moment conditions in Assumptions~\ref{ass:envelope} and~\ref{ass:f_lipschitz} hold under both $\mathbb P_0$ and $\mathbb P_1$. Assume the three samples are independent as described above and $n,m',m''\to\infty$. Then
\[
\left(\big|\widehat\Delta-\Delta^\star\big|-\delta_{\rm inv}\right)_+
=o_p(1).
\]
In particular, if $\delta_{\rm inv}=0$, then $\widehat\Delta \xrightarrow{p} \Delta^\star.$ Moreover, if the local inverse-stability condition in Assumption~\ref{assm:error_bound_P0} holds with exponent $\kappa\ge 1$, then
\[
\begin{aligned}
\left(\big|\widehat\Delta-\Delta^\star\big|-\delta_{\rm inv}\right)_+
=O_p\big(&n^{-1/(2\kappa)}+(m')^{-1/2}\\
&+(m'')^{-1/2}\big),
\end{aligned}
\]
where $(x)_+:=\max\{x,0\}$.
\end{theorem}

Theorem~\ref{thm:delta_hat_rate} shows that $\widehat\Delta$ consistently estimates the oracle gap up to the unavoidable ambiguity $\delta_{\rm inv}$, with separate rates for preference estimation, benchmark construction, and evaluation noise.

\begin{theorem}[Asymptotic validity and power of the test]
\label{thm:power_rate}
Suppose the assumptions of Theorem~\ref{thm:delta_hat_rate} and the local inverse-stability condition in Assumption~\ref{assm:error_bound_P0} hold, and suppose the conditional CLT in Assumption~\ref{ass:eval_clt} holds. Let
\[
a_n:=n^{-1/(2\kappa)}
\]
denote the preference-estimation error rate. If $\delta_{\rm inv}=0$ and $\sqrt{m''}\,a_n\to0$, then the Wald test in \eqref{eq:T_wald} has asymptotic level $\alpha$:
\[
\limsup_{n,m',m''\to\infty}
\Pr\{\mathrm{reject}\ H_0\}
\le \alpha
\qquad
\text{whenever } \Delta^\star\le \tau .
\]
Without this relative-rate condition, the same conclusion holds on any fixed-margin null $\Delta^\star\le\tau-\eta$, $\eta>0$. More generally, if $\delta_{\rm inv}>0$, the unadjusted test controls Type-I error on the separated null $\Delta^\star\le\tau-\delta_{\rm inv}-\eta$. Under the alternative, for any fixed $\eta>0$, if $\Delta^\star\ge \tau+\delta_{\rm inv}+\eta$, then
\[
\lim_{n,m',m''\to\infty}
\Pr\{\mathrm{reject}\ H_0\} = 1.
\]
\end{theorem}
Accordingly, the unadjusted test has a margin guarantee when $\delta_{\rm inv}>0$. If a valid sensitivity bound $\bar\delta_{\rm inv}\ge\delta_{\rm inv}$ is available, replacing $\tau$ by $\tau+\bar\delta_{\rm inv}$ in \eqref{eq:T_wald}, under the same relative-rate and CLT conditions, yields a conservative sensitivity-adjusted test of the original null $\Delta^\star\le\tau$.

\section{Experiments}
\label{sec:experiments}
We evaluate \texttt{RADAR} in three settings: controlled synthetic optimization problems, a semi-synthetic capacity-allocation task, and a real-world police districting study. The synthetic and capacity allocation experiments are formulated as two-sample decision-adequacy tests, where we can construct target distributions with known oracle optimality gaps and therefore label each shift as decision-irrelevant $(H_0)$ or decision-relevant $(H_1)$. The police districting study evaluates \texttt{RADAR} in a temporal monitoring setting, where the goal is to detect when accumulated changes in operational data make the incumbent districting plan inadequate.

Across all settings, we compare \texttt{RADAR} against three alternative tests: ($i$) X-Mean, a two-sample test for changes in the context mean; ($ii$) X-Distr, a two-sample test for any distributional change of context $X$; and ($iii$) Risk-Value, a decision-related baseline~\citep{podkopaev2022tracking} that tracks changes in the realized risk of the deployed decision $z_0$ using a plug-in estimate $\hat\theta$ and the same risk functional as the forward problem. We report empirical rejection rates over repeated trials, corresponding to Type-I error under $H_0$ and power under $H_1$, at significance level $\alpha=0.05$. We set the optimality-gap tolerance to $\tau=0$ in synthetic experiments, where the oracle gap is exactly computable, and use a relative tolerance $\tau=0.1$ in the semi-synthetic experiments. The districting study, which monitors a decision revised at high switching cost, uses $\tau=0.2$ at level $\alpha=0.01$; Appendix~\ref{app:atlanta} reports its sensitivity to both. Inverse optimization is implemented using \texttt{cvxpylayers}; complete implementation details and hyperparameters are provided in Appendix~\ref{app:config}.

\vspace{.1in}
\noindent\emph{\ul{Synthetic Experiments}}.
We first evaluate \texttt{RADAR} on two stochastic optimization problems: a quadratic program (QP) with expectation risk and a linear program (LP) with CVaR risk. They test different notions of decision relevance. In the expectation setting, harmful shifts are those that change the mean structure of the induced cost vector $X\theta$; in the CVaR setting, harmful shifts are those that alter the adverse tail scenarios driving the risk objective. Here we describe the QP with expectation risk example to illustrate the experimental setup; the equivalent descriptions of the LP example can be found in the Appendix~\ref{app:experiment}.
\begin{figure}[!t]
    \centering
    \includegraphics[width=\columnwidth]{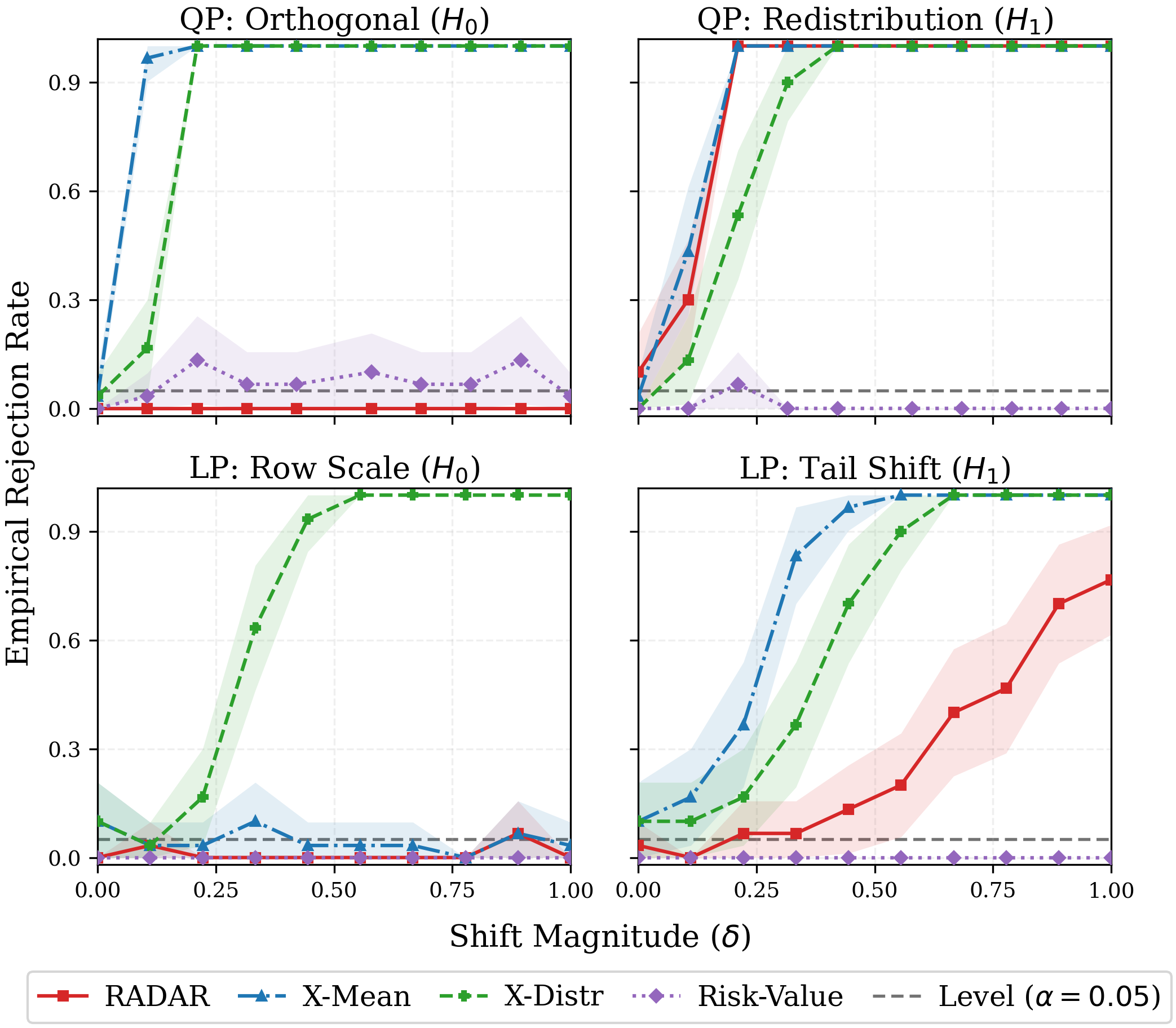}
    \caption{Empirical rejection rates on synthetic LP and QP examples across different decision-irrelevant ($H_0$) and decision-relevant ($H_1$) shift settings. Each panel reports the rejection rate at level $\alpha=0.05$ as shift magnitude $\delta$ varies. \texttt{RADAR} stays near the nominal level under $H_0$, demonstrating Type-I error control, and gains competitive power under $H_1$. In contrast, baselines often reject in both regimes, leading to inflated Type-I error under $H_0$.}
    \label{fig:synthetic_power}
    \vspace{-.2em}
\end{figure}
Let $X\in\mathbb R^{n\times d}$ denote a random context matrix, whose rows are sampled independently from Dirichlet distributions, $X_i \sim \mathrm{Dir}(\alpha_i)$ with parameter $\alpha_i\in\mathbb R_+^d$, for $i=1,\dots,n$. The latent preference parameter $\theta\in\Delta^{d-1}$ lies on the probability simplex, $\Delta^{d-1}:= \left\{\theta\in\mathbb R^d_+:\sum_{j=1}^d \theta_j=1\right\}.$ The forward QP problem with expectation risk is
\[
    z^\star(\mathbb P,\theta)
    \in
    \arg\min_{z\in\Delta^{n-1}}
    \mathbb E_{\mathbb P}
    \left[
        \frac12 z^\top Qz + \langle X\theta, z\rangle
    \right].
\]
Intuitively, each row of $X$ represents an item with $d$ features; $\theta$ maps these features to item cost vector $X\theta\in\mathbb R^n$; and $z\in\Delta^{n-1}$ allocates fractional resources across the $n$ items.

Context distribution shifts are generated by perturbing the Dirichlet parameter matrix $\Alpha\in\mathbb R_+^{n\times d}$, whose $i$th row is $\alpha_i$, to a shifted matrix $\Alpha'(\delta)$, where $\delta$ controls the shift magnitude level. This construction creates both decision-irrelevant shifts $(H_0)$ and decision-relevant shifts $(H_1)$. For example, an orthogonal perturbation sets $\alpha_i'=\alpha_i+\delta v$ for a selected row, with $\langle v,\theta^\star\rangle=0$, so the context distribution changes while the induced cost is unchanged for any $\delta$ (decision-irrelevant). In contrast, a row-redistribution perturbation sets $\alpha'_{i,j_1}=\alpha_{i,j_1}+\delta$ and $\alpha'_{i,j_2}=\alpha_{i,j_2}-\delta$, preserving the row sum but changing the row-wise mean direction, which can induce a nonzero optimality gap (decision-relevant). Each shifted regime is labeled by its oracle gap of whether $\Delta(z_0;\theta^\star,\mathbb P_1)\le \tau$. Figure~\ref{fig:synthetic_power} reports empirical rejection rates across representative shift types. Under $H_0$, \texttt{RADAR} remains close to the significance level as $\delta$ increases, indicating controlled Type-I error. Under $H_1$, its power increases rapidly with $\delta$, tracking the growth of the oracle optimality gap. In contrast, generic context two-sample tests and risk-value baselines often reject in both regimes, leading to inflated Type-I error under decision-irrelevant shifts. Full optimization formulations, shift definitions, and implementation details are deferred to Appendix~\ref{app:synthetic_experiment}.

\begin{figure}[!t]
    \centering
    \includegraphics[width=\linewidth]{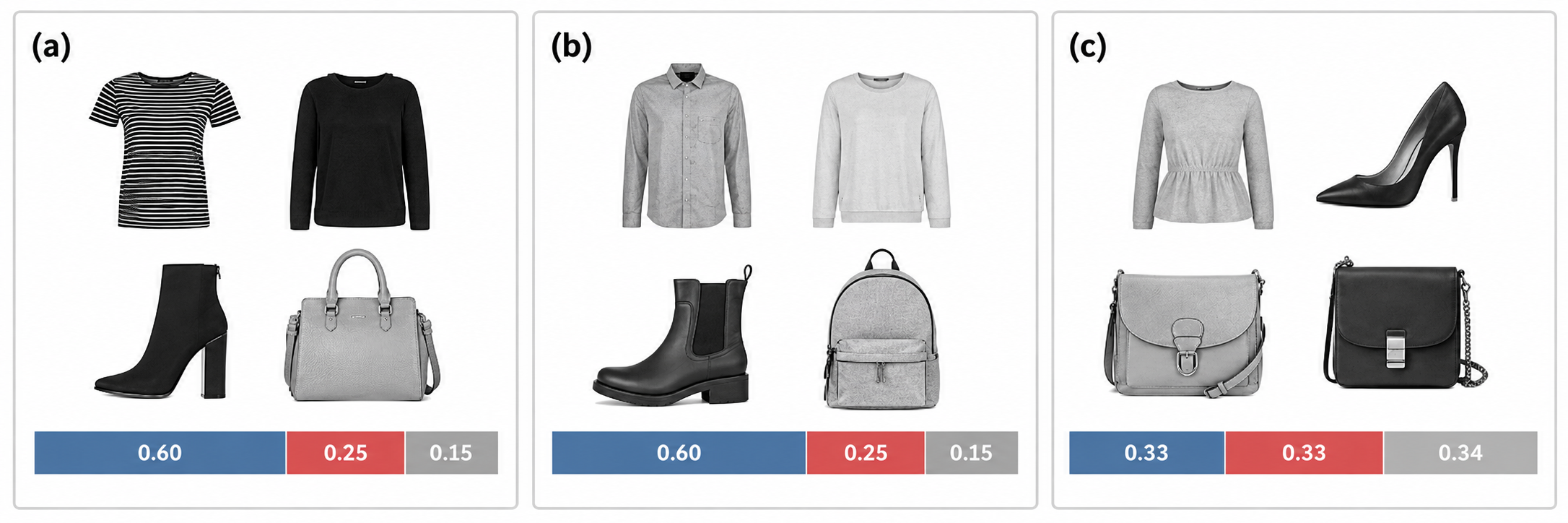}
\caption{Fashion-MNIST capacity-allocation regimes. Each panel shows four representative images of the regime, and the induced coarse-group shares bar. The stacked bars report, from left to right, the proportions of clothing, footwear, and accessories. \textbf{(a)} Baseline, share dominated by clothing. \textbf{(b)} A decision-irrelevant $H_0$ shift replaces categories within the clothing group, producing visually different images while preserving the coarse-group shares. \textbf{(c)} A decision-relevant $H_1$ shift changes the group shares, warranting re-optimization. Panels \textbf{(a)} and \textbf{(b)} are therefore statistically distinguishable but operationally equivalent.}
\label{fig:fashion_mnist_demo}
\end{figure}

\vspace{.1in}
\noindent\emph{\ul{Multi-Product Newsvendor Capacity Allocation}}. 
We next study a semi-synthetic capacity-allocation task built from Fashion-MNIST~\citep{xiao2017fashion}. We treat each image as a unit of demand belonging to one of three product groups (clothing, footwear, accessory) according to its Fashion-MNIST class label. The decision maker commits a fixed capacity share (e.g., shelf space) across groups before demand is realized, balancing groupwise shortage against overage costs whose weights are the latent preferences $\theta$. The images serve as high-dimensional demand contexts, so that shifts in the image distribution induce shifts in demand that may or may not be decision-relevant. Appendix \ref{app:fashion_mnist} gives the details of the class mapping, exact loss, and forward optimization problem.

Figure~\ref{fig:fashion_mnist_demo} shows the shift constructions. The baseline group shares are $(0.60,0.25,0.15)$. Under $H_0$, the fine image classes change while the group shares, and hence the forward objective, remain fixed; the image distribution moves but the deployed allocation remains optimal. We consider two decision-relevant $H_1$ shifts. The balanced $H_1$ shift changes the shares to $(1/3,1/3,1/3)$, which is visible in both the image distribution and the deployed risk. In contrast, the orthogonal $H_1$ shift changes the optimal allocation in a direction with little impact on the realized cost, isolating decision adequacy from deployed risk. The signal level $\delta\in[0,1]$ interpolates from the baseline to each endpoint.

Table~\ref{tab:fashion_mnist_partial_results} reports representative $H_0$ and orthogonal-$H_1$ results; full results appear in Appendix Table~\ref{tab:fashion_mnist_full_results}. Each baseline fails under some path: under $H_0$, the context test baselines increasingly reject a harmless image shift, whileboth Risk-Value and X-Mean miss the orthogonal $H_1$ shift. By directly targeting the optimality gap, \texttt{RADAR} distinguishes visible distribution or deployed risk changes from decision inadequacy.

\begin{table}[!t]
\centering
\vspace{-.1em}

\resizebox{1\linewidth}{!}{
\begin{tabular}{llcccc}
\toprule
\textbf{Shift} & \textbf{Level $\delta$} & \textbf{\texttt{RADAR}} & \textbf{X-Mean} & \textbf{X-Distr} & \textbf{Risk-Value} \\
\midrule
$H_0$ & $0.0$ & $0.02 \pm 0.04$ & $0.06 \pm 0.06$ & $0.02 \pm 0.04$ & $0.06 \pm 0.06$ \\
$H_0$ & $0.3$ & $0.04 \pm 0.05$ & $0.44 \pm 0.11$ & $0.90 \pm 0.07$ & $0.04 \pm 0.05$ \\
$H_0$ & $0.6$ & $0.04 \pm 0.05$ & $0.90 \pm 0.07$ & $1.00 \pm 0.03$ & $0.12 \pm 0.08$ \\
$H_0$ & $0.9$ & $0.06 \pm 0.06$ & $1.00 \pm 0.03$ & $1.00 \pm 0.03$ & $0.06 \pm 0.06$ \\
\midrule
$H_1$ & $0.0$ & $0.00 \pm 0.03$ & $0.04 \pm 0.05$ & $0.02 \pm 0.04$ & $0.04 \pm 0.05$ \\
$H_1$ & $0.3$ & $0.14 \pm 0.08$ & $0.04 \pm 0.05$ & $0.06 \pm 0.06$ & $0.02 \pm 0.04$ \\
$H_1$ & $0.6$ & $0.38 \pm 0.11$ & $0.04 \pm 0.05$ & $0.28 \pm 0.10$ & $0.10 \pm 0.07$ \\
$H_1$ & $0.9$ & $0.68 \pm 0.11$ & $0.14 \pm 0.08$ & $0.66 \pm 0.11$ & $0.04 \pm 0.05$ \\
\bottomrule
\end{tabular}
}
\caption{Representative Fashion-MNIST results for the image-only $H_0$ shift and the orthogonal harmful $H_1$ shift. Entries report the empirical rejection rate $\pm$ the half-width of the 90\% Wilson confidence interval across $50$ independent repetitions at level $\alpha=0.05$.}
\label{tab:fashion_mnist_partial_results}
\vspace{-.2em}
\end{table}

\begin{figure}[!t]
\centering
\begin{subfigure}{.32\linewidth}
\centering
\includegraphics[width=\linewidth]{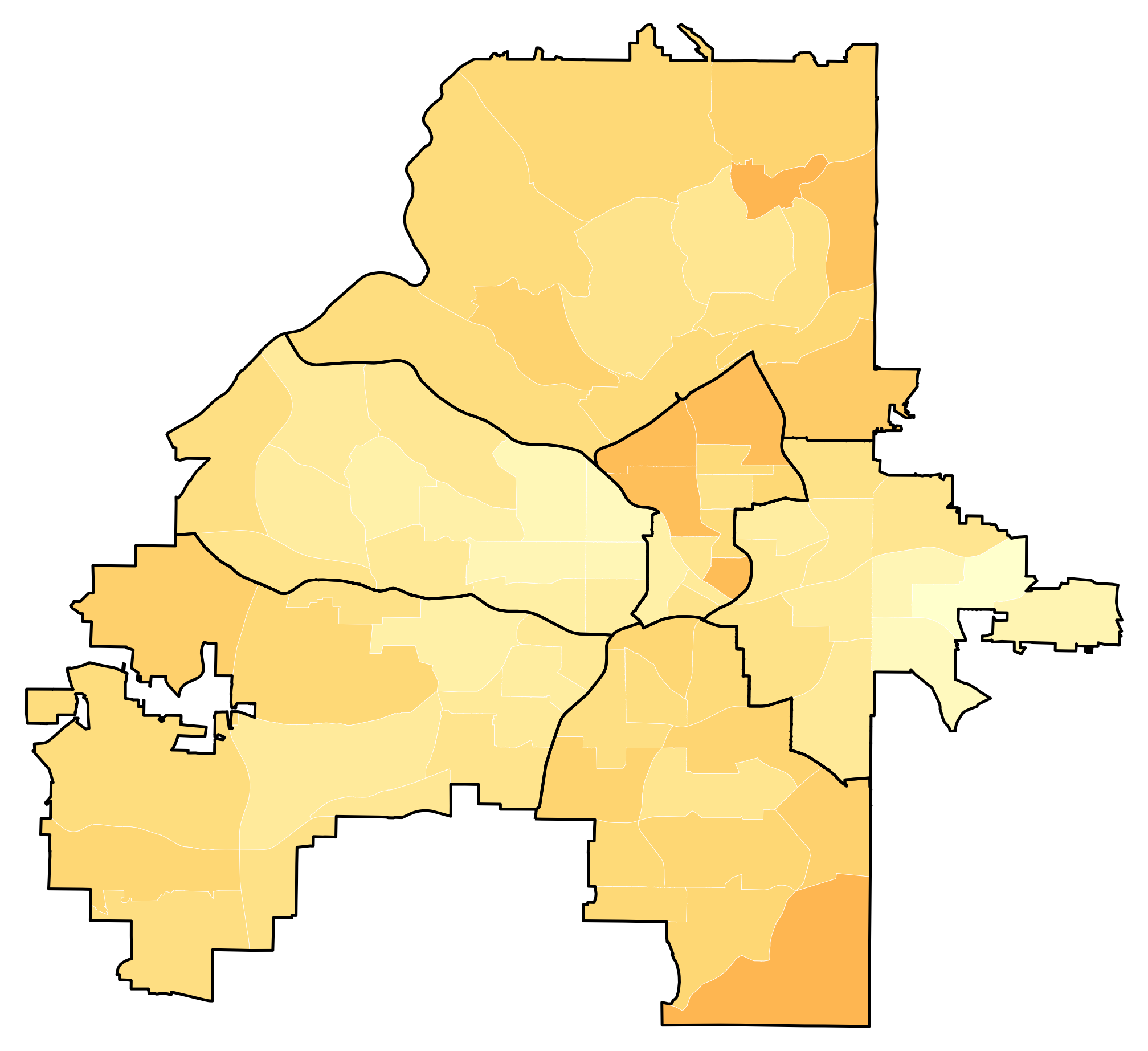}
\caption{2014}
\end{subfigure}
\begin{subfigure}{.32\linewidth}
\centering
\includegraphics[width=\linewidth]{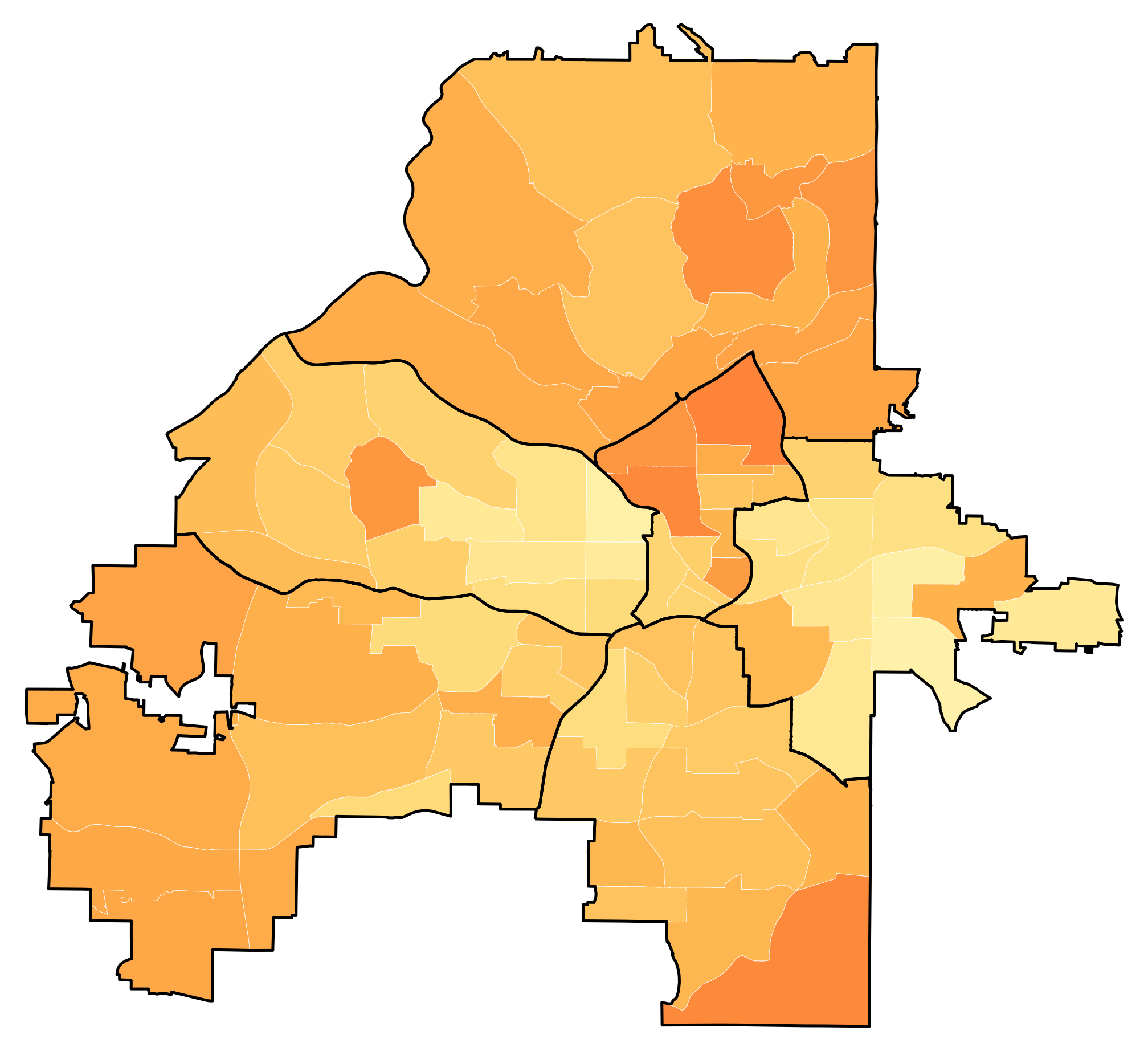}
\caption{2015}
\end{subfigure}
\begin{subfigure}{.32\linewidth}
\centering
\includegraphics[width=\linewidth]{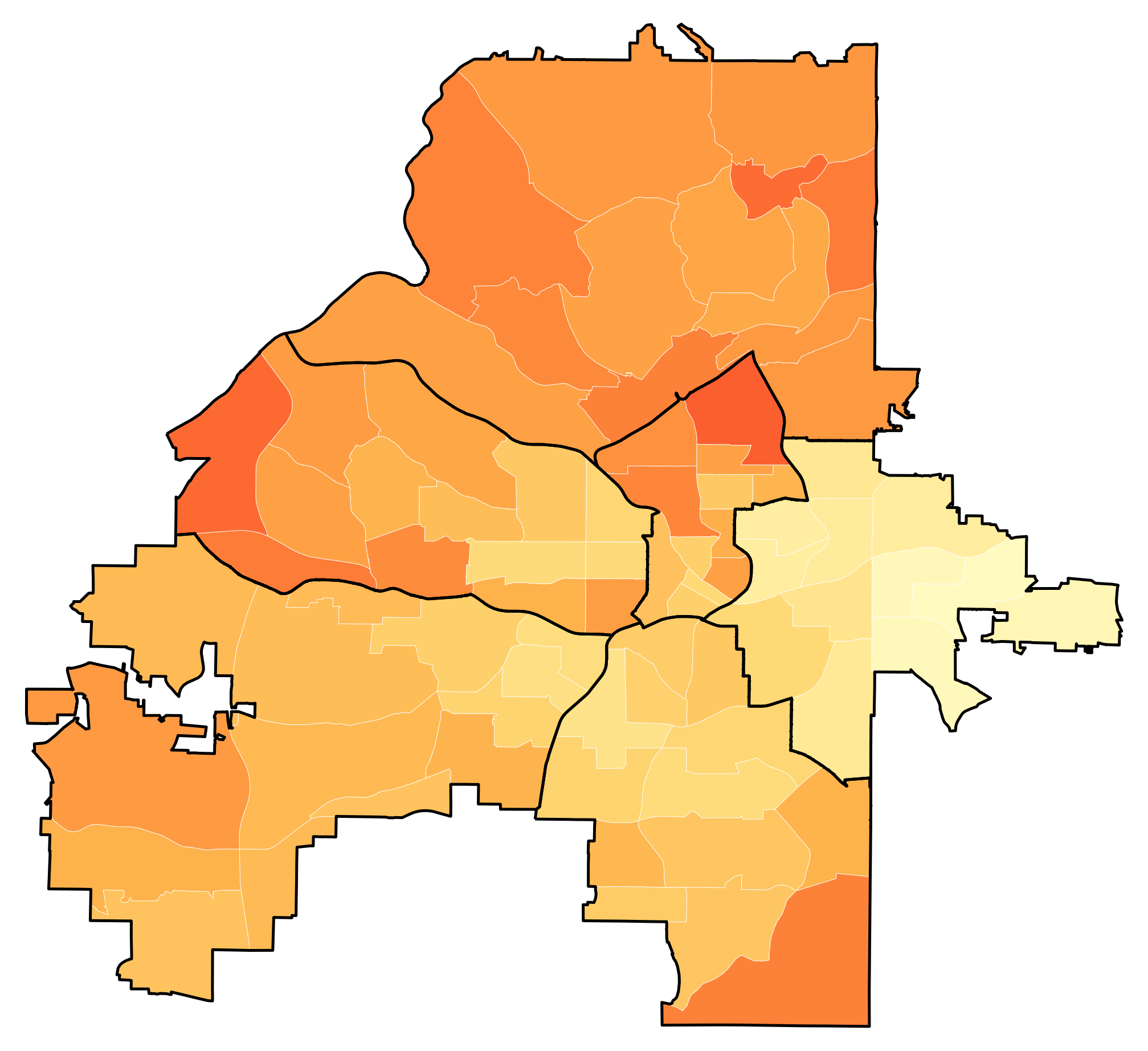}
\caption{2016}
\end{subfigure}
\begin{subfigure}{.32\linewidth}
\centering
\includegraphics[width=\linewidth]{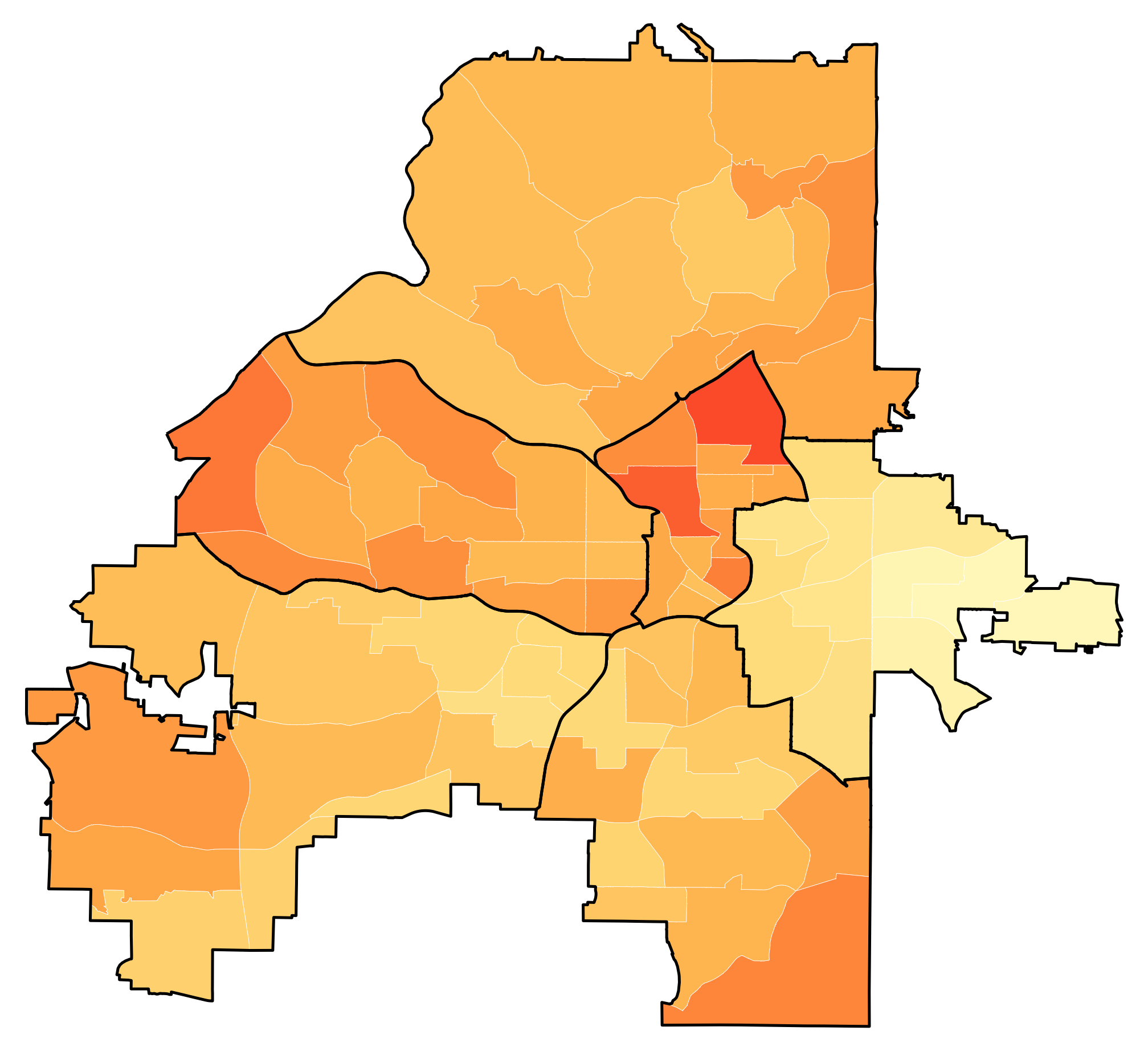}
\caption{2017}
\end{subfigure}
\begin{subfigure}{.32\linewidth}
\centering
\includegraphics[width=\linewidth]{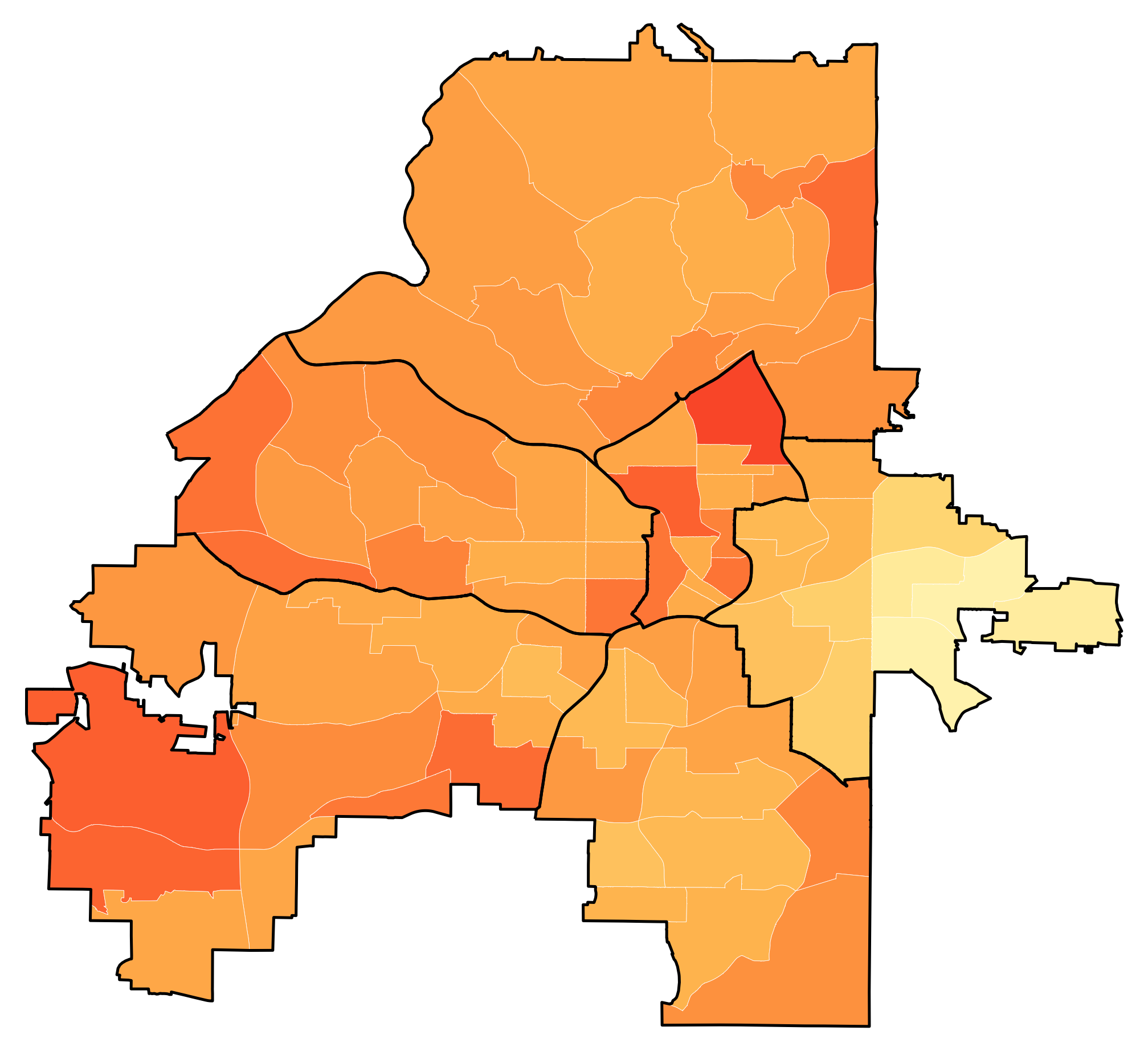}
\caption{2018}
\end{subfigure}
\begin{subfigure}{.32\linewidth}
\centering
\includegraphics[width=\linewidth]{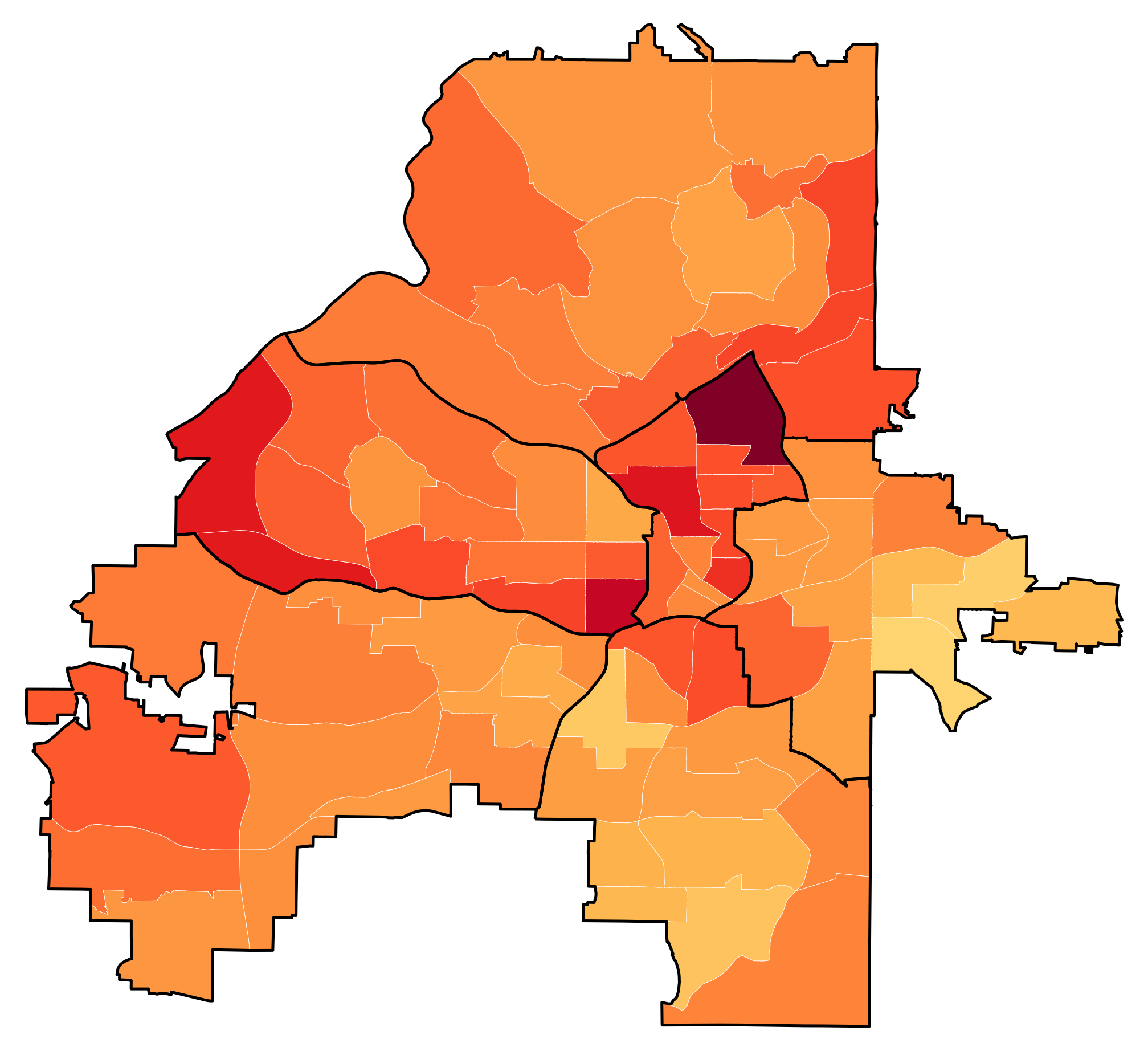}
\caption{2019}
\end{subfigure}
\begin{subfigure}{\linewidth}
\centering
\includegraphics[width=\linewidth]{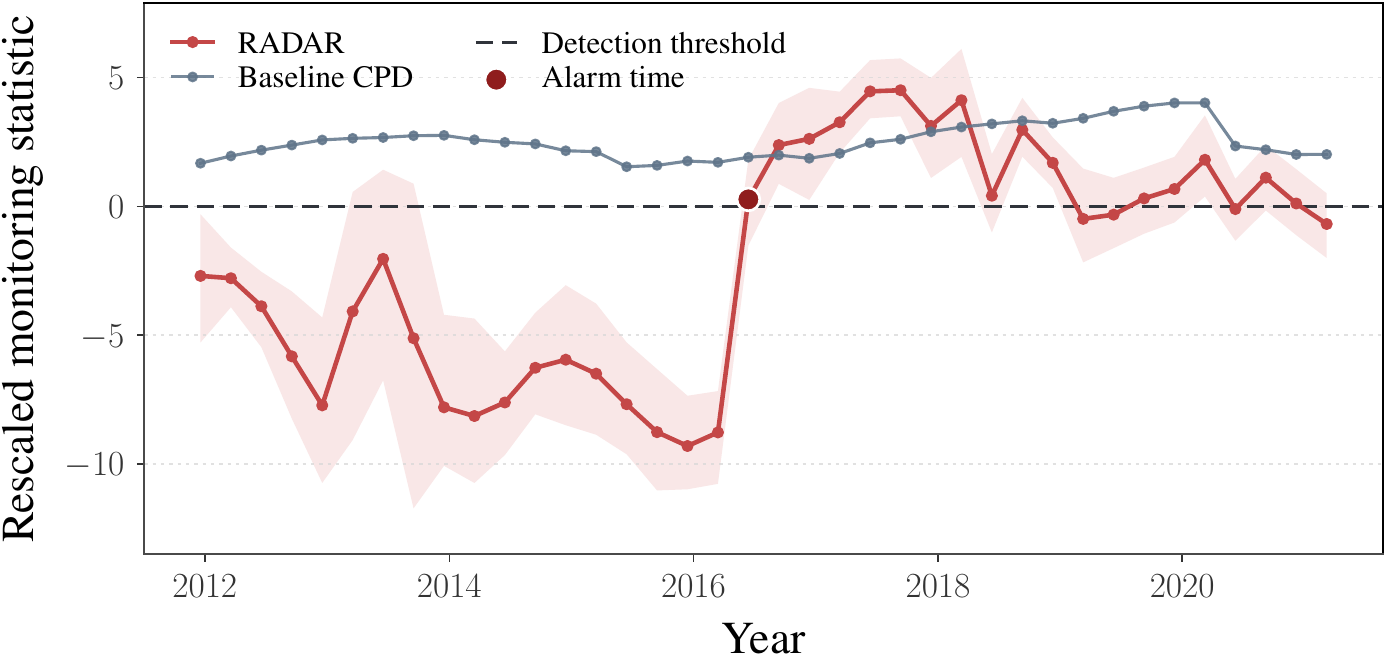}
\caption{Sequential monitoring results.}
\end{subfigure}
\caption{Analysis of the Atlanta police districting task. \textbf{(a)--(f)} Spatial distributions of annual mean response times from 2014 to 2019; darker colors indicate longer response times, and black outlines delineate patrol zones. \textbf{(g)} Decision-focused sequential monitoring of the incumbent districting plan, with the scan statistic peaking around June 2016.}
  \label{fig:atlanta_combined}
\vspace{-.1in}
\end{figure}

\vspace{.1in}
\noindent\emph{\ul{Atlanta Police Districting}}.
Our final study audits a real deployed decision with a known service life. The Atlanta Police Department partitions the city into $78$ geographic \emph{beats}, the atomic units of patrol assignment, and groups them into $6$ \emph{zones}; a districting plan assigns beats to zones subject to zone-size and contiguity constraints. Such a plan is costly to revise, since it relocates personnel and command structure, and is revised correspondingly rarely: the plan we audit took effect in December $2011$ and remained in force until March $2019$, when the department replaced it, citing uneven workload and response times across zones. Its planning objective is to equalize operational load across zones, which we model as a variance-balancing problem: a good assignment makes zone-level averages of a beat-week load signal comparable. From incident timestamps we construct two contexts per beat and week, a response time $r^{\mathrm{resp}}_{ti}$ and a workload $r^{\mathrm{work}}_{ti}$; the planner's trade-off between them is unknown, so we work with the combined signal $r^{(\theta)}_{ti} = \theta r^{\mathrm{resp}}_{ti} + (1-\theta) r^{\mathrm{work}}_{ti}$ and recover $\theta\in[0,1]$ by inverse optimization. The environment drifts continuously while the plan is held fixed, so the operative question is not whether the environment changed, but whether it changed enough to render the incumbent plan inadequate---exactly the sequential monitoring setting \texttt{RADAR} is designed for. Full optimization details, feasibility constraints, and implementation are deferred to Appendix~\ref{app:atlanta}.

Figure~\ref{fig:atlanta_combined}\textbf{(a)--(f)} illustrates the gradual spatial evolution of annual mean response times. Although these panels reveal temporal changes in the operational context, distributional change alone does not establish whether the incumbent districting plan has become inadequate. Figure~\ref{fig:atlanta_combined}\textbf{(g)} reports the sequential monitoring path described in Appendix~\ref{app:atlanta}. The \texttt{RADAR} statistic remains far below the detection threshold for consecutive monitoring times spanning the plan's first nearly five years in service, then first crosses it in June $2016$ and rejects at subsequent monitoring times, indicating a consistent pattern of incumbent decision suboptimality rather than a transient excursion. The baseline change-point statistic, by contrast, exceeds its detection threshold at \emph{every} monitoring time, beginning with the first and including the entire pre-alarm period in which the incumbent's estimated optimality gap stays within the tolerance.

We read June $2016$ as a model-based early-warning signal. Separate analyses support our finding. \citet{zhu2022data} report that workload variance across zones rose by $17.98\%$ in $2017$ relative to $2016$ and deteriorated further without re-configuration. By March $2018$, patrols in Zone~$2$ were publicly reported as strained by call volume relative to the cars available to cover its thirteen beats.\footnote{\emph{Atlanta Journal-Constitution}, \url{https://www.ajc.com/news/local/torpy-large-not-enough-apd-cops-stop-shoplifters-that-crime/cwsV4R389yeeLMXxRXehjN/}.} A new zone plan was eventually adopted in March $2019$, which reduced workload imbalance by $43\%$ and high-priority response time by $5.8\%$, confirming that a materially better plan had indeed been available.\footnote{Atlanta Police Department news release, \url{https://www.atlantapd.org/Home/Components/News/News/190/}.} The alarm thus precedes the measured deterioration by roughly a year and the re-configuration by two and a half, which is the ordering an early-warning monitor should produce. The baseline's behavior is the empirical counterpart of the distinction drawn in related work: the operating environment is always changing, so a test of distributional change---or, as in \citet{podkopaev2022tracking}, of risk alone---is nearly uninformative about whether a deployed decision still warrants its place, whereas targeting the optimality gap separates the one change that mattered from the continuous drift around it.

\vspace{.1in}
\noindent\emph{\ul{Robustness to Preference Non-Identification}}.
We further examine whether the audit conclusion is stable when $\theta^\star$ is only set-identified. Rather than fixing a single rationalizing $\hat\theta$, we compute the population range $[\underline\Delta_{\rm inv},\overline\Delta_{\rm inv}]$ of the target gap over the entire inverse-feasible set $\Theta_{\rm inv}$ and read the conclusion off its position relative to $\tau$: the incumbent is adequate if $\overline\Delta_{\rm inv}\le\tau$, re-optimization is warranted if $\underline\Delta_{\rm inv}>\tau$, and the conclusion is indeterminate otherwise. We demonstrate with an example in a linear program whose incumbent is a vertex, so that $\Theta_{\rm inv}$ is a full-dimensional cone rather than a point, decision-irrelevant shifts give $[\underline\Delta_{\rm inv},\overline\Delta_{\rm inv}]=[0.000,0.000]$ and the conclusion is determinate despite non-identification, whereas decision-relevant shifts give $\underline\Delta_{\rm inv}>\tau$ once their magnitude is large enough and leave it indeterminate below that magnitude. Setup, shift definitions, and the full ranges are deferred to Appendix~\ref{app:io_ambiguity} and Table~\ref{tab:delta_inv}.

\section{Discussion and Future Work}
\texttt{RADAR} assesses whether distributional change has rendered a deployed decision suboptimal, rather than whether the context distribution has merely changed. Its guarantees are conditional on the specified forward model: under misspecification of the objective family, the estimated optimality gap need not correspond to the realized operational loss. The inverse step further requires the incumbent to be informative about $\theta^\star$, since weakly rationalized decisions widen $[\underline\Delta_{\rm inv},\overline\Delta_{\rm inv}]$; when this range straddles $\tau$, the audit is indeterminate rather than underpowered. Challenger construction may be computationally demanding, although its error is one-sided and affects power rather than validity. A non-rejection is accordingly inconclusive, being consistent with adequacy, preference ambiguity, or limited evaluation data. Promising directions include auditing policies observed across multiple contexts, calibrating $\tau$ to switching costs, and sequential monitoring under explicit multiplicity control.

\bibliography{arxiv/ref}

\appendix
\clearpage
\onecolumn
\section{Details of the Testing Procedures}
\label{app:test_details}

We provide additional details on the testing procedures used to approximate the sampling distribution of the estimated deployment optimality gap $\widehat\Delta$. Throughout this section, the optimization and parameter estimation steps are treated as fixed conditional on the first-stage samples, and inference is carried out using only the independent evaluation subsample. This sample-splitting construction separates benchmark construction from gap evaluation and allows standard one-sided testing of
\[
H_0:\Delta^\star \le \tau
\qquad \text{versus} \qquad
H_1:\Delta^\star > \tau .
\]
We consider two approaches.

\begin{enumerate}

    \item \textbf{Wald test.}
    Under standard regularity conditions, the evaluation-sample estimator admits the asymptotic expansion
    \[
    \sqrt{m''}\bigl(
    \widehat\Delta-\Delta(z_0;\hat\theta,\mathbb P_1)
    \bigr)
    \overset{d}{\longrightarrow}
    \mathcal N(0,\sigma^2),
    \]
    for some $\sigma^2>0$. Let $\hat\sigma^2$ be a consistent estimator of $\sigma^2$ computed from the evaluation subsample. We define the one-sided Wald statistic
    \begin{equation}
    \label{eq:wald_stat_step3}
        T_{\rm Wald}
        =
        \frac{\sqrt{m''}(\widehat\Delta-\tau)}{\hat\sigma}.
    \end{equation}
    The Wald test rejects $H_0$ at level $\alpha$ when $T_{\rm Wald}>z_{1-\alpha}$, where $z_{1-\alpha}$ is the $(1-\alpha)$ quantile of the standard normal distribution. Equivalently, the one-sided $p$-value is
    \begin{equation}
    \label{eq:p_wald}
        p_{\rm Wald}
        =
        1-\Phi(T_{\rm Wald}),
    \end{equation}
    where $\Phi$ denotes the standard normal CDF.

    \item \textbf{Bootstrap test.}
    For general risk functionals, small evaluation samples, or nonsmooth risks such as VaR and CVaR, the normal approximation may be less accurate. We therefore also consider a nonparametric bootstrap test based on the evaluation subsample.

    Draw $B$ bootstrap samples of size $m''$ with replacement from the evaluation subsample. For each bootstrap replicate $b=1,\dots,B$, recompute the estimated gap in \eqref{eq:gap_hat_step3} on the resampled data, while keeping $\hat\theta$ and $\hat z_1$ fixed. This gives bootstrap realizations $\{\widehat\Delta_b^\ast\}_{b=1}^B$. Define the empirical distribution of the centered bootstrap statistic by
    \[
    \widehat F_B(t)
    =
    \frac{1}{B}
    \sum_{b=1}^B
    \mathbf{1}
    \left\{
    \sqrt{m''}\bigl(\widehat\Delta_b^\ast-\widehat\Delta\bigr)
    \le t
    \right\}.
    \]
    The one-sided bootstrap $p$-value is
    \begin{equation}
    \label{eq:p_boot}
        p_B
        =
        1-\widehat F_B
        \left(
        \sqrt{m''}(\widehat\Delta-\tau)
        \right).
    \end{equation}
    We reject $H_0$ when $p_B<\alpha$.
\end{enumerate}

The Wald test is computationally efficient and is appropriate when the evaluation sample estimator is well approximated by its asymptotic normal limit. The bootstrap test is more computationally demanding but is often preferable for small samples, nonsmooth risk functionals, or general risk objectives for which closed-form variance estimation is inconvenient. In both cases, sample splitting ensures that the learned preference parameter, the constructed challenger decision, and the final evaluation of the deployment gap are statistically separated.

\section{Decision-Relevant Change-Point Detection}
\label{subsec:cpd}

Classical change-point detection (CPD) asks whether the context distribution changes.  \texttt{RADAR} instead asks whether the post-change distribution makes the deployed decision $z_0$ inadequate.  For a tolerance $\tau\ge0$, define the decision-safe set
\[
\mathcal S_\tau(z_0)
\;:=\;
\Big\{
\mathbb P \;\big|\; \Delta(z_0;\theta^\star,\mathbb P)\le \tau
\Big\},
\]
the distributions under which $z_0$ is $\tau$-optimal.  We assume the stream starts from a decision-safe baseline, $\mathbb P_0\in\mathcal S_\tau(z_0)$, and may change once at an unknown first post-change index $\gamma$.  This gives three regimes:
\begin{align}
\label{eq:regime_nochange}
\mathcal H_0:\quad
&X_i \sim \mathbb P_0, \qquad i=1,\dots,T,\\[2mm]
\label{eq:regime_safechange}
\mathcal H_{\mathrm{safe}}:\quad
&\exists\,\gamma \ \text{s.t.}\ 
X_i \sim 
\begin{cases}
\mathbb P_0, & i < \gamma,\\
\mathbb P_1, & i \ge \gamma,
\end{cases}
\quad\text{and}\quad
\mathbb P_0\ne\mathbb P_1,\ \mathbb P_1\in \mathcal S_\tau(z_0),\\[2mm]
\label{eq:regime_unsafechange}
\mathcal H_{\mathrm{unsafe}}:\quad
&\exists\,\gamma \ \text{s.t.}\ 
X_i \sim 
\begin{cases}
\mathbb P_0, & i < \gamma,\\
\mathbb P_1, & i \ge \gamma,
\end{cases}
\quad\text{and}\quad
\mathbb P_0\ne\mathbb P_1,\ \mathbb P_1\notin\mathcal S_\tau(z_0).
\end{align}
Decision-relevant monitoring tests
\begin{equation}
\label{eq:dr_cpd_test}
H_0:\ \mathcal H_{0} \cup \mathcal H_{\mathrm{safe}},
\qquad \text{vs.} \qquad
H_1:\ \mathcal H_{\mathrm{unsafe}}.
\end{equation}
By comparison, standard CPD tests
\begin{equation}
\label{eq:standard_cpd_test}
H_0^{\mathrm{CPD}}:\ \mathcal H_{0} ,
\qquad \text{vs.} \qquad
H_1^{\mathrm{CPD}}:\ \mathcal H_{\mathrm{safe}}
\cup \mathcal H_{\mathrm{unsafe}}.
\end{equation}
Thus, standard CPD flags either kind of change, whereas \texttt{RADAR} alarms only when the stream leaves $\mathcal S_\tau(z_0)$. Figure~\ref{fig:cpd_motivation} illustrates the distinction.  A variance-only shift at $\gamma_1=1000$ changes the distribution but leaves $z_0$ adequate; a mean shift at $\gamma_2=2000$ makes it inadequate.  Standard CPD reacts to both shifts, and its rejection rate decays back to the nominal level once each change has passed out of its comparison windows.  \texttt{RADAR} reacts only to the second and keeps rejecting thereafter, because its estimated gap remains above the tolerance for as long as the stream stays in the unsafe regime.
\begin{figure*}[t]
    \centering
    \includegraphics[width=\textwidth]{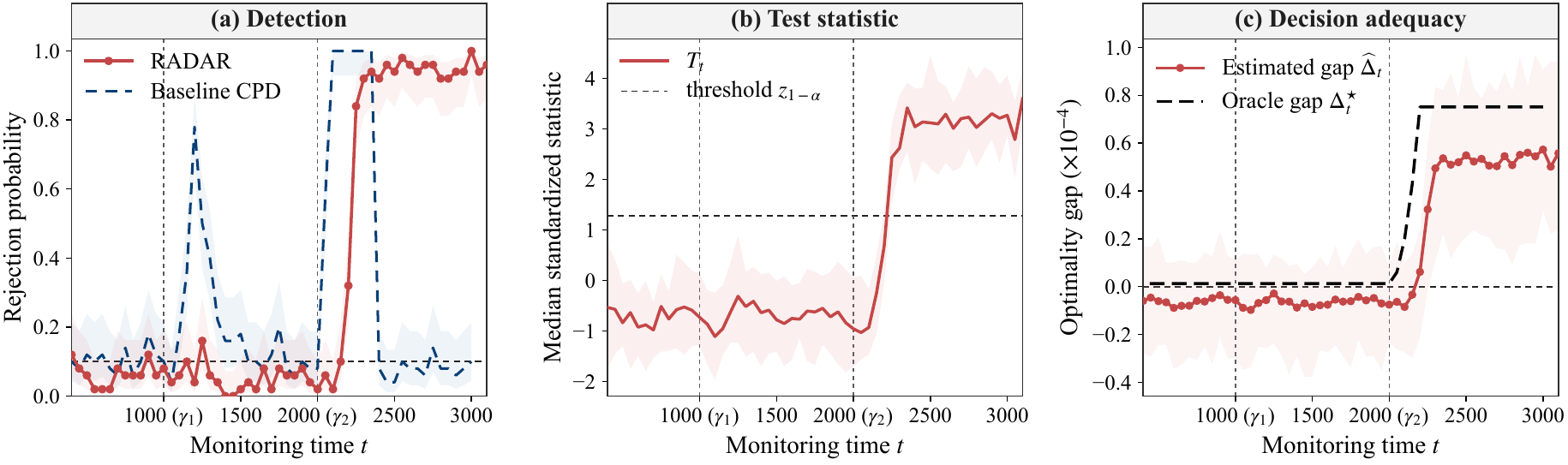}
    \caption{Decision-relevant sequential monitoring with two changes: the shift at $\gamma_1=1000$ is decision-irrelevant, whereas the shift at $\gamma_2=2000$ is decision-relevant. Every statistic at monitoring time $t$ uses the $w=200$ observations ending at $t$, so neither method can respond before $t=\gamma+w$. \textbf{(a)} Standard CPD flags both changes and then forgets each one, returning to its nominal level; \texttt{RADAR} ignores $\gamma_1$ and, once the incumbent becomes inadequate at $\gamma_2$, keeps rejecting for the remainder of the record. \textbf{(b)} The alarm statistic $T_t$ of \eqref{eq:local_alarm_statistic} against the threshold $z_{1-\alpha}$. \textbf{(c)} The estimated and oracle window optimality gaps against the tolerance $\tau=0$. Shading denotes pointwise $95\%$ intervals in \textbf{(a)} and \textbf{(c)} and the interquartile range across replications in \textbf{(b)}.}
    \label{fig:cpd_motivation}
\end{figure*}

\subsection{Single-window monitoring and the alarm time}
\label{subsec:cpd_procedure}

We assume a prespecified burn-in period $\mathcal I_0$ known to contain no change. We estimate $\hat\theta_0$ once from $(\{X_i:i\in\mathcal I_0\},z_0)$ by the inverse problem \eqref{eq:iop} and hold $(\hat\theta_0,z_0)$ fixed throughout monitoring. Freezing the preference parameter serves two purposes: it keeps $\hat\theta_0$ independent of the monitored stream, so the evaluation-stage inference at each monitoring time conditions on a first stage estimated from separate data, and it puts the statistics at different monitoring times on a common scale, which a re-estimated $\hat\theta_t$ would not. Any scale normalization used to express $\tau$ on a relative basis is fixed on $\mathcal I_0$ for the same reason.

Let $\mathcal T=\{t_1<\cdots<t_N\}$ denote the monitoring times. Each uses the trailing window of the $w$ most recent observations,
\begin{equation}
\label{eq:trailing_window}
\mathcal W_t=\{t-w+1,\ldots,t\},
\qquad
\mathcal W_t\cap\mathcal I_0=\emptyset ,
\end{equation}
so that no monitoring window reuses burn-in data. Write $\mathbb P_t$ for the population distribution of $\mathcal W_t$ and
\begin{equation}
\label{eq:window_optimality_gap}
\Delta_t:=\Delta(z_0;\theta^\star,\mathbb P_t)
\end{equation}
for the optimality gap of the incumbent under the current operating conditions. As in Section~\ref{sec:hypo-test}, we split $\mathcal W_t$ at random into a benchmark-construction split of size $w'$ and an evaluation split of size $w''$, with $w'+w''=w$. The benchmark split yields a challenger $\hat z_t$ by the construction in Step~2, and the evaluation split yields the held-out estimate $\widehat\Delta_t$ of \eqref{eq:window_optimality_gap} together with the evaluation-stage standard deviation $\hat\sigma_t$, whose standard error is $\hat\sigma_t/\sqrt{w''}$.

The local hypotheses at $t$ are the decision-adequacy hypotheses \eqref{eq:hypothesis_test} applied to the current window,
\begin{equation}
\label{eq:local_alarm_test}
H_{0,t}:\ \Delta_t\le\tau
\qquad\text{versus}\qquad
H_{1,t}:\ \Delta_t>\tau ,
\end{equation}
and the corresponding statistic is the Wald statistic \eqref{eq:T_wald} evaluated on the window,
\begin{equation}
\label{eq:local_alarm_statistic}
T_t
=\frac{\sqrt{w''}\,(\widehat\Delta_t-\tau)}{\hat\sigma_t}.
\end{equation}

\paragraph{Alarm time.}
With $N=|\mathcal T|$ monitoring time points and the Bonferroni critical value $q:=z_{1-\alpha/N}$, we define
\begin{equation}
\label{eq:alarm_time}
\hat t_{\rm alarm}
:=
\min\{t\in\mathcal T:\ T_t>q\},
\qquad \min\emptyset:=\infty ,
\end{equation}
and raise an alarm at $\hat t_{\rm alarm}$, declaring that the incumbent decision warrants re-examination. Since $\mathcal W_t$ is trailing, the statistic indexed by $t$ is available at time $t$ itself; the price is that an alarm lags the onset of inadequacy by at most $w$ observations. We do not report $\argmax_{t\in\mathcal T}T_t$: the operational question is when the incumbent should be re-examined, not where in the past the environment moved, and under sustained inadequacy the maximizer is determined by the noise in an already-alarmed regime.

\begin{algorithm}[!tbh]
\caption{Decision-relevant sequential monitoring}
\label{alg:dr_cpd}
\begin{algorithmic}[1]
\STATE \textbf{Input:} Contexts $\{X_i\}_{i\ge1}$; burn-in set $\mathcal I_0$; deployed decision $z_0$; window width $w$ with split sizes $(w',w'')$, $w'+w''=w$; monitoring times $\mathcal T$; tolerance $\tau$; level $\alpha$.
\STATE Estimate $\widehat\theta_0$ from $(\{X_i:i\in\mathcal I_0\},z_0)$ by the inverse problem \eqref{eq:iop}; hold $(\widehat\theta_0,z_0)$ and the scale normalization fixed throughout.
\STATE Set the Bonferroni critical value $q\gets z_{1-\alpha/|\mathcal T|}$.
\FOR{$t\in\mathcal T$}
    \STATE Form the trailing window $\mathcal W_t$ by \eqref{eq:trailing_window}.
    \STATE Split $\mathcal W_t$ into benchmark and evaluation split of sizes $(w',w'')$.
    \STATE Construct $\widehat z_t$ on the benchmark split; compute $(\widehat\Delta_t,\widehat\sigma_t)$ and $T_t$ by \eqref{eq:local_alarm_statistic} on the evaluation split.
    \IF{$T_t>q$}
        \STATE \textbf{return} alarm at $t$.
    \ENDIF
\ENDFOR
\STATE \textbf{return} no alarm.
\end{algorithmic}
\end{algorithm}

\subsection{Theoretical guarantees}
\label{subsec:cpd_theory}

The monitoring grid is fixed: $\mathcal T$ contains $N$ equally spaced times with stride $s$, and asymptotics are taken in the sample sizes, with $N$, $s$ and $\alpha$ held fixed. This is the relevant regime for a deployed monitor, where the horizon and the review schedule are set by the application and what grows is the data behind each review.

Because the split of $\mathcal W_t$ is uniformly at random, the evaluation split targets the \emph{window mixture}
\begin{equation}
\label{eq:window_mixture}
\mathbb P_t:=\frac1w\sum_{i\in\mathcal W_t}\mathbb P_{(i)} ,
\end{equation}
where $\mathbb P_{(i)}$ is the law of $X_i$. A window lying in a single regime has $\mathbb P_t$ equal to that regime's law; a window straddling the change at $\gamma$ has $\mathbb P_t=\lambda_t\mathbb P_0+(1-\lambda_t)\mathbb P_1$ with $\lambda_t$ the fraction of the window before $\gamma$. Everything below rests on one condition.

\begin{assumption}[Local validity]
\label{ass:cpd_local}
For every $t\in\mathcal T$ with $\Delta_t\le\tau$ and every $c\in\mathbb R$,
\[
\limsup\ \Pr\big(T_t>c\big)\ \le\ 1-\Phi(c).
\]
\end{assumption}

Assumption~\ref{ass:cpd_local} is the window-level form of the guarantee already established for the two-sample test, and it holds under the conditions used there. Applying Assumptions~\ref{ass:compact}--\ref{ass:f_lipschitz} and~\ref{ass:eval_clt} to $\mathcal W_t$ makes the evaluation-stage statistic asymptotically standard normal once the burn-in error is negligible, $\sqrt{w''}\,\sup_{t}|\Delta(z_0;\hat\theta_0,\mathbb P_t)-\Delta_t|\xrightarrow{p}0$, which holds when $\delta_{\rm inv}=0$ and $w''=o(n^{1/\kappa})$ under Assumption~\ref{assm:error_bound_P0}. On a window straddling a change the evaluation fold is drawn without replacement from independent but non-identically distributed observations, which inflates $\hat\sigma_t$ relative to the true standard error and makes $T_t$ conservative; the one-sided domination in Assumption~\ref{ass:cpd_local} is therefore preserved.

\begin{lemma}[The decision-safe set is convex]
\label{lem:safe_convex}
For fixed $z_0$ and $\theta$, the map $\mathbb P\mapsto\Delta(z_0;\theta,\mathbb P)$ is convex, so $\mathcal S_\tau(z_0)$ is convex. Consequently, under $H_0$ in \eqref{eq:dr_cpd_test}, every window mixture satisfies $\mathbb P_t\in\mathcal S_\tau(z_0)$, and hence $\Delta_t\le\tau$ for all $t\in\mathcal T$.
\end{lemma}

\begin{proof}
$\Delta(z_0;\theta,\mathbb P)=\E_{\mathbb P}[f(z_0;X,\theta)]-\inf_{z\in\mathcal Z}\E_{\mathbb P}[f(z;X,\theta)]$ is a linear functional of $\mathbb P$ minus an infimum of linear functionals, hence linear plus convex, hence convex; $\mathcal S_\tau(z_0)$ is one of its sublevel sets. Under $\mathcal H_0$ every $\mathbb P_t$ equals $\mathbb P_0$. Under $\mathcal H_{\rm safe}$, $\mathbb P_t$ is by \eqref{eq:window_mixture} a mixture of $\mathbb P_0$ and $\mathbb P_1$, both in $\mathcal S_\tau(z_0)$, so $\mathbb P_t\in\mathcal S_\tau(z_0)$ by convexity.
\end{proof}

Lemma~\ref{lem:safe_convex} settles the failure mode a windowed monitor is most exposed to. A window straddling a decision-irrelevant change mixes two safe regimes, and one might worry the incumbent is inadequate for the mixture though adequate for each part. Convexity says this cannot happen, so the transition itself cannot provoke an alarm.

\begin{theorem}[No false alarm over the horizon]
\label{thm:cpd_type1}
Suppose Assumption~\ref{ass:cpd_local} holds and let $q=z_{1-\alpha/N}$. Then, under $H_0$ in \eqref{eq:dr_cpd_test},
\[
\limsup\ \Pr_{H_0}\big(\hat t_{\rm alarm}<\infty\big)\ \le\ \alpha .
\]
\end{theorem}

\begin{proof}
By Lemma~\ref{lem:safe_convex}, $\Delta_t\le\tau$ at every $t\in\mathcal T$, so Assumption~\ref{ass:cpd_local} applies at each with $c=q$ and gives $\Pr(T_t>q)\le\alpha/N+o(1)$. The alarm event is $\bigcup_{t\in\mathcal T}\{T_t>q\}$, and Boole's inequality bounds its probability by $N(\alpha/N)+o(1)$, since $N$ is fixed.
\end{proof}

The companion result asks not whether an alarm occurs but when. It needs no change-point structure: all that matters is that the incumbent is inadequate at some monitoring time by a margin the evaluation split can resolve.

\begin{theorem}[The monitor alarms by $t^\star$]
\label{thm:cpd_power}
Suppose there is $t^\star\in\mathcal T$ with $\varepsilon:=\Delta_{t^\star}-\tau>0$, let $\sigma_{t^\star}$ be the limiting evaluation-stage standard deviation at $t^\star$, and suppose the studentized estimation error is bounded in probability with a diverging margin,
\[
\frac{\sqrt{w''}\,\varepsilon}{\sigma_{t^\star}}-q\ \longrightarrow\ \infty .
\]
Then $\Pr\big(\hat t_{\rm alarm}\le t^\star\big)\to1$.
\end{theorem}

\begin{proof}
Write $T_{t^\star}$ as the studentized estimation error plus $\sqrt{w''}(\Delta_{t^\star}-\tau)/\hat\sigma_{t^\star}$. The first part is $O_p(1)$ and the second diverges faster than $q$, so $\Pr(T_{t^\star}>q)\to1$. Since $\hat t_{\rm alarm}$ is the first crossing, $\{T_{t^\star}>q\}\subseteq\{\hat t_{\rm alarm}\le t^\star\}$.
\end{proof}

\begin{corollary}[Delay after an abrupt change]
\label{cor:cpd_delay}
Under $H_1$ in \eqref{eq:dr_cpd_test} with first post-change index $\gamma$, let $t_\gamma$ be the first monitoring time with $t_\gamma-w+1\ge\gamma$, and assume the horizon reaches it. Its window lies entirely after the change, so $\Delta_{t_\gamma}=\Delta(z_0;\theta^\star,\mathbb P_1)$. If the margin condition of Theorem~\ref{thm:cpd_power} holds at $t_\gamma$, then
\[
\Pr\big(\hat t_{\rm alarm}-\gamma\le w+s-2\big)\longrightarrow1 .
\]
\end{corollary}

\begin{proof}
Apply Theorem~\ref{thm:cpd_power} at $t^\star=t_\gamma$. The grid is equally spaced with stride $s$, so its first element at or beyond $\gamma+w-1$ satisfies $t_\gamma\le\gamma+w-1+(s-1)$.
\end{proof}

\section{Inverse Optimization Ambiguity}
\label{app:io_ambiguity}

Definition~\ref{def:inverse_sensitivity} treats $\delta_{\rm inv}$ as a target-domain sensitivity index rather than an identification assumption or finite-sample error. Point identification is sufficient, but not necessary, for $\delta_{\rm inv}=0$: a non-singleton inverse-feasible set can still induce a unique target-domain gap. For example, in the strongly convex quadratic problem
\[
\min_z \ \frac12 z^\top Qz-\theta^\top z,\qquad Q\succ0,
\]
the optimizer satisfies $z^\star(\theta)=Q^{-1}\theta$, so observing $z_0$ uniquely identifies $\theta=Qz_0$.

In contrast, point identification can fail in constrained or piecewise-linear decision problems. Consider the linear program
\[
\min_{z\in\mathcal Z}\ \theta^\top z,
\]
where $\mathcal Z$ is a polytope. If the observed decision $z_0$ is a vertex, then every objective vector in the normal cone
\[
N_{\mathcal Z}(z_0)
:=
\{\theta:\theta^\top(z-z_0)\ge 0,\ \forall z\in\mathcal Z\}
\]
rationalizes $z_0$ as an optimal solution. Hence the inverse problem identifies a set of compatible preference parameters rather than a single parameter. This does not necessarily invalidate the adequacy test: if all $\theta\in N_{\mathcal Z}(z_0)$, or all statistically plausible rationalizing parameters, induce nearly the same target-domain optimality gap $\Delta(z_0;\theta,\mathbb P_1)$, then the decision adequacy conclusion is stable despite non-identification.

\paragraph{Checkable conditions for $\delta_{\rm inv}=0$.}
Three sufficient conditions can be verified rather than assumed. \textup{($i$)} \emph{Point identification}: $\Theta_{\rm inv}$ is a singleton, for which full column rank of $M_0$ after scale normalization is sufficient in the interior quadratic model $\E_{\P_0}f=\tfrac12z^\top Qz+\langle M_0\theta,z\rangle$ with $Q\succ0$. \textup{($ii$)} \emph{Functional identification}: point identification is not necessary, since $\delta_{\rm inv}=0$ whenever $\Delta(z_0;\cdot,\P_1)$ is constant on $\Theta_{\rm inv}$, which a non-singleton $\Theta_{\rm inv}$ does not preclude. Under expectation risk with $f$ affine in $X$, so that $\E_{\P}[f]$ depends on $\P$ only through $\E_{\P}[X]$, this holds for any shift preserving $\E_{\P}[X]$, and for objectives positively homogeneous in the cost vector also for a positive rescaling of it. This case is specific to expectation risk: a mean-preserving change of dispersion leaves the expectation objective fixed but moves a $\CVaR$ objective, so $\delta_{\rm inv}$ must be recomputed for tail risks. \textup{($iii$)} \emph{Richer observed behavior}: if $K$ decisions are observed under distinct baseline environments sharing one $\theta^\star$, then $\Theta_{\rm inv}=\bigcap_{k\le K}\Theta_{\rm inv}^{(k)}$, so $\delta_{\rm inv}$ is nonincreasing in $K$. Conversely, $\delta_{\rm inv}>0$ should be expected when $z_0$ is a vertex of a polytope, where $\Theta_{\rm inv}$ is a normal cone with nonempty interior.

More generally, $[\underline\Delta_{\rm inv},\overline\Delta_{\rm inv}]$ describes how the target-domain adequacy conclusion varies across preference parameters that are observationally equivalent under the baseline decision. Thus, the downstream question can be well identified even when the preference parameter itself is not. In controlled synthetic problems, this range can be computed directly and is often zero up to numerical tolerance. In larger problems, exact optimization over $\Theta_{\rm inv}$ may be difficult. A practical sensitivity analysis instead forms the approximate rationalizing set
\[
\widehat\Theta_{\rm inv}(\varepsilon)
:=
\left\{\theta\in\Theta:
\Delta(z_0;\theta,\widehat{\mathbb P}_0)\le\varepsilon\right\}
\]
and reports the range of target-gap estimates obtained from multiple rationalizing parameters, generated, for example, by varying initialization, regularization, or the inverse-feasibility tolerance. If this empirical range lies entirely below or above $\tau$, the update conclusion is insensitive to the selected rationalizer; otherwise, the result should be reported as identification-sensitive. A fully robust implementation replaces the plug-in gap by lower and upper optimizations over $\widehat\Theta_{\rm inv}(\varepsilon)$, or uses a justified upper bound $\bar\delta_{\rm inv}$ to calibrate the test conservatively.

\vspace{.1in}
\noindent\emph{\ul{Robustness to Preference Non-Identification}}.
The preceding analysis considers a single rationalizing $\hat\theta$; we now study whether the audit conclusion is stable across $\theta\in\Theta_{\rm inv}$. In an LP with expectation risk ($\tau=0.005$), $z_0$ is a vertex, so $\Theta_{\rm inv}$ is the normal cone at $z_0$ and covers $40.6\%$ of the preference simplex; we compute the population range $[\underline\Delta_{\rm inv},\overline\Delta_{\rm inv}]$ over it and report in Table~\ref{tab:delta_inv}. A decision-irrelevant shift that changes only dispersion, leaving $\E_{\P_1}[X]=\E_{\P_0}[X]$, yields $\underline\Delta_{\rm inv}=\overline\Delta_{\rm inv}=0$: the expectation objective is unchanged, so $z_0$ remains optimal under \emph{every} $\theta\in\Theta_{\rm inv}$ and the target gap vanishes identically on $\Theta_{\rm inv}$ rather than only at $\theta^\star$. The sensitivity index is then exactly zero, $\delta_{\rm inv}=0$, and the adequacy conclusion is determinate despite $\theta^\star$ being only set-identified. Decision-relevant shifts instead raise the cost of the selected item or lower that of a rejected one, and the range is then wide: at small magnitudes it straddles $\tau$ and the conclusion is indeterminate, but once $\underline\Delta_{\rm inv}>\tau$, re-optimization is warranted under every $\theta\in\Theta_{\rm inv}$.

\paragraph{Experimental setup for Table~\ref{tab:delta_inv}.}
The instance is a linear program with expectation risk over $K=6$ candidate decisions and $d=3$ preference dimensions. Rows of the cost matrix $M_0\in\mathbb R^{6\times 3}$ lie on the simplex, and effective costs are $M^{\rm eff}=\mathrm{diag}(s)M_0$ with heterogeneous item scales $s=(1.00,1.02,1.06,1.10,1.12,1.15)$. The scales are not cosmetic: with every row on the simplex each decision costs exactly $1/d$ at the uniform preference, so the uniform $\theta$ would rationalize every decision, lie in the closure of every rationalizing cone, and force $\underline\Delta_{\rm inv}=0$ by construction. Contexts are Dirichlet with concentration $60\,M_0$. The true preference is $\theta^\star=(0.60,0.30,0.10)$ and the incumbent $z_0$ is the LP-optimal vertex at $\theta^\star$. Because $z_0$ is a vertex, $\Theta_{\rm inv}$ is the normal cone at $z_0$, which occupies $40.6\%$ of the preference simplex by exact area computation. The tolerance is $\tau=0.02\cdot(M^{\rm eff}_{z_0})^\top\theta^\star=0.005$, i.e.\ two percent of the incumbent's own expected cost. We use $n=m=2000$ and $\alpha=0.05$ over $300$ replications.

\paragraph{Shift families.}
Writing $\delta\ge 0$ for the shift magnitude, we consider three decision-irrelevant families and two decision-relevant ones. \textup{($i$)} \emph{Dispersion only} divides all Dirichlet concentrations by $1+\delta$, leaving means unchanged and moving only dispersion. \textup{($ii$)} \emph{Uniform cost rise} scales every decision's cost by $1+\delta$, a uniform cost increase that raises the incumbent's realized risk without creating a better alternative. \textup{($iii$)} \emph{Competitor worsens} scales one competitor's cost up by $1+\delta$, making the incumbent only more attractive. \textup{($iv$)} \emph{Incumbent worsens} scales the incumbent's own cost up by $1+\delta$, so it becomes inadequate and its realized risk rises. \textup{($v$)} \emph{Rival improves} scales one competitor's cost down by $1-\delta$, so a strictly better decision appears while the incumbent's realized risk is unchanged. Two of these are diagnostic: \emph{Uniform cost rise} and \emph{Rival improves} are exactly the two cells a risk-tracking monitor misclassifies, since realized risk moves without the decision becoming inadequate in the first, and stays fixed while it does in the second.

\paragraph{Computing the identified range.}
$\underline\Delta_{\rm inv}$ and $\overline\Delta_{\rm inv}$ are population quantities obtained exactly rather than by grid search. Writing $A=M^{\rm eff}_{z_0}-M^{\rm eff}$, the upper end is $\overline\Delta_{\rm inv}=\max_j\max_{\theta\in\Theta_{\rm inv}}A_j^\top\theta$, one linear program per decision; the lower end $\underline\Delta_{\rm inv}=\min_{\theta\in\Theta_{\rm inv}}\max_j A_j^\top\theta$ is a single linear program in epigraph form.

\begin{table}[!t]
\centering
\small

\vspace{-.25em}
\setlength{\tabcolsep}{4pt}
\begin{tabular}{@{}llcl@{}}
\toprule
Shift family & $\delta$ & $[\underline\Delta_{\rm inv},\overline\Delta_{\rm inv}]$ & Decision \\
\midrule
Decision-irrelevant ($H_0$)
& all & $[0.000,0.000]$ & Adequate \\
\addlinespace
Incumbent worsens ($H_1$)
& $0.250$ & $[0.000,0.088]$ & Indeterminate \\
Incumbent worsens ($H_1$)
& $0.950$ & $[0.084,0.335]$ & Re-optimize \\
\addlinespace
Rival improves ($H_1$)
& $0.290$ & $[0.000,0.099]$ & Indeterminate \\
Rival improves ($H_1$)
& $0.762$ & $[0.127,0.267]$ & Re-optimize \\
\bottomrule
\end{tabular}
\caption{Population identified ranges $[\underline\Delta_{\rm inv},\overline\Delta_{\rm inv}]$ over $\Theta_{\rm inv}$, with $\tau=0.005$. The incumbent is adequate if $\overline\Delta_{\rm inv}\le\tau$, re-optimization is warranted if $\underline\Delta_{\rm inv}>\tau$, and the conclusion is indeterminate if $\underline\Delta_{\rm inv}\le\tau<\overline\Delta_{\rm inv}$.}
\label{tab:delta_inv}

\vspace{-.25em}
\end{table}

\section{Theoretical Proofs}

\label{appendix:theory_proof}
In this section, we provide detailed proofs for the main theorems and several supporting lemmas. We establish consistency and convergence rates for each error component in \eqref{eq:delta_decomp}, derive the resulting consistency and convergence rate of the test statistic, and then prove the stated hypothesis-testing guarantees.

\paragraph{General risk notation.}
Throughout this section, we work with the general risk-based objective
\[
R_{\mathbb P}(z,\theta)
:=\rho^{\mathbb P}\!\left(f(z;X,\theta)\right)
\]
and the corresponding optimality gap
\[
\Delta(z;\theta,\rho,\mathbb P)
:=R_{\mathbb P}(z,\theta)
-\inf_{z'\in\mathcal Z}R_{\mathbb P}(z',\theta).
\]
The expectation-risk setting is recovered by taking $\rho^{\mathbb P}(Y)=\mathbb E_{\mathbb P}[Y]$. The results cover the risk functionals satisfying Assumption~\ref{ass:risk_class}, including expectation risk and OCE/CVaR-type risks. Unless otherwise specified, observations within each domain or split are independent, and the samples used for preference estimation, benchmark construction, and evaluation are independent.

For a fixed $\rho$, define
\[
\Theta_{\rm inv}(z_0;\rho,\mathbb P_0)
:=\left\{\theta\in\Theta:
z_0\in\argmin_{z\in\mathcal Z}R_{\mathbb P_0}(z,\theta)
\right\}.
\]
Within this section, write $\Theta_{\rm inv}$ for this risk-specific inverse set and define
\[
\underline\Delta_{\rm inv}
:=\inf_{\theta\in\Theta_{\rm inv}}
\Delta(z_0;\theta,\rho,\mathbb P_1),
\qquad
\overline\Delta_{\rm inv}
:=\sup_{\theta\in\Theta_{\rm inv}}
\Delta(z_0;\theta,\rho,\mathbb P_1),
\]
with $\delta_{\rm inv}:=\overline\Delta_{\rm inv} -\underline\Delta_{\rm inv}$. The estimators $\hat\theta$ and $\hat z_1$ denote the risk-based counterparts of \eqref{eq:iop} and~\eqref{eq:z_1_hat_method}. Also define
\[
\widehat\Delta_\rho
:=R_{\widehat{\mathbb P}_1''}(z_0,\hat\theta)
-R_{\widehat{\mathbb P}_1''}(\hat z_1,\hat\theta),
\]
and abbreviate $R_1(z,\theta):=R_{\mathbb P_1}(z,\theta)$ and $\widehat R_1''(z,\theta):=R_{\widehat{\mathbb P}_1''}(z,\theta)$. The exact decomposition is
\begin{align*}
\widehat\Delta_\rho-\Delta(z_0;\theta^\star,\rho,\mathbb P_1)
&=\underbrace{\Delta(z_0;\hat\theta,\rho,\mathbb P_1)
-\Delta(z_0;\theta^\star,\rho,\mathbb P_1)}_{\text{($i$) preference}}\\
&\quad+\underbrace{\inf_{z\in\mathcal Z}R_1(z,\hat\theta)
-R_1(\hat z_1,\hat\theta)}_{\text{($ii$) benchmark}}\\
&\quad+\underbrace{\bigl\{\widehat R_1''(z_0,\hat\theta)
-R_1(z_0,\hat\theta)\bigr\}
-\bigl\{\widehat R_1''(\hat z_1,\hat\theta)
-R_1(\hat z_1,\hat\theta)\bigr\}}_{\text{($iii$) evaluation}}.
\end{align*}
For expectation risk, $\widehat\Delta_\rho=\widehat\Delta$ and this reduces to \eqref{eq:delta_decomp}. The lemmas below use this general-risk interpretation of terms~\textup{($i$)}--\textup{($iii$)}.

\subsection{Assumptions}
We introduce several standard assumptions on the parameter space, objective function, and risk functional that will be used in the theoretical analysis. First, we impose a structural assumption on the parameter and decision spaces.

\begin{assumption}[Compactness]
\label{ass:compact}
The parameter space $\Theta$ and the decision space $\mathcal Z$ are compact subsets of finite-dimensional Euclidean spaces.
\end{assumption}

This assumption is standard in statistical learning and optimization. It holds, for example, when the feasible set is a bounded polyhedron or a compact convex set. Compactness ensures existence of minimizers and controls the complexity of the induced function class.

We next introduce regularity conditions on the objective and the risk functional.

\begin{assumption}[Regularity of the objective function]
\label{ass:f_regular}
The objective function
\[
f:\mathcal Z \times \mathcal X \times \Theta \to \mathbb R
\]
is measurable in $x$ and continuous in $(z,\theta)$ for almost every $x$.
\end{assumption}

This condition is satisfied by many objectives of practical interest. For instance, it holds for linear objectives
\[
f(z,x,\theta)=c(x,\theta)^\top z
\]
and quadratic objectives
\[
f(z,x,\theta)=\tfrac12 z^\top Q(x,\theta) z + q(x,\theta)^\top z,
\]
provided the coefficients depend continuously on $(x,\theta)$.

\begin{assumption}[Envelope condition]
\label{ass:envelope}
There exists a measurable function $F:\mathcal X\to\mathbb R_+$ such that
\[
\sup_{z\in\mathcal Z,\ \theta\in\Theta} |f(z,X,\theta)| \le F(X)
\quad \text{a.s.},
\]
and
\[
\E[F(X)^2]<\infty.
\]
\end{assumption}

This is a standard integrability condition in empirical process theory. It is automatically satisfied when $f$ is uniformly bounded. More generally, for linear objectives of the form $f(z,X,\theta)=c(X,\theta)^\top z$, compactness of $\mathcal Z$ together with
\[
\E\!\left[\sup_{\theta\in\Theta}\|c(X,\theta)\|\right]<\infty
\]
implies the existence of such an envelope.

\begin{assumption}[Lipschitz continuity of $f$]
\label{ass:f_lipschitz}
There exists a measurable function $L:\mathcal X\to\mathbb R_+$ such that
\[
\E[L(X)^2]<\infty
\]
and, for all $z,z'\in\mathcal Z$ and $\theta,\theta'\in\Theta$,
\[
|f(z,X,\theta)-f(z',X,\theta')|
\le
L(X)\bigl(\|z-z'\|+\|\theta-\theta'\|\bigr)
\quad \text{a.s.}
\]
\end{assumption}

This condition holds for many finite-dimensional objective families, including linear and quadratic objectives with coefficients depending smoothly on $(X,\theta)$. It provides a convenient way to control the complexity of the induced function class and derive uniform convergence rates.

\begin{assumption}[Optimized certainty equivalent risk]
\label{ass:risk_class}
We assume that the risk measure belongs to the class of optimized certainty equivalents (OCE). Specifically, for any integrable random variable $Y$ under probability law $\P$, the risk can be expressed as
\[
\rho^\P(Y)
=
\inf_{m\in\mathcal M}
\left\{
m+\E_\P[\ell(Y-m)]
\right\},
\]
where $\mathcal M\subset\R$ is a compact set, and $\ell:\R\to\R$ is measurable, continuous, and Lipschitz with constant $L_\ell$.
\end{assumption}

Assumption~\ref{ass:risk_class} focuses on a standard and tractable class of convex risk measures. It includes expectation as the special case obtained by taking $\ell(t)=t$ and $\mathcal M=\{0\}$, in which case
\[
\rho^\P(Y)=\E_\P[Y].
\]
It also includes conditional value-at-risk (CVaR) at level $\alpha\in(0,1)$ through the choice
\[
\ell_\alpha(t)=\frac{1}{1-\alpha}(t)_+,
\]
which gives
\[
\mathrm{CVaR}_\alpha^\P(Y)
=
\inf_{m\in\mathcal M}
\left\{
m+\frac{1}{1-\alpha}\E_\P[(Y-m)_+]
\right\},
\]
where $\mathcal M$ is any compact interval containing the optimizer.

\begin{assumption}[Local inverse-stability condition]
\label{assm:error_bound_P0}
Let $\Theta_{\rm inv}:=\Theta_{\rm inv}(z_0;\rho,\mathbb P_0)$ and $\mathrm{dist}(\theta,S):=\inf_{\theta'\in S}\|\theta-\theta'\|_2$. There exist constants $C>0$, $\epsilon>0$, and a degree $\kappa \ge 1$ such that for all $\theta\in\Theta$, if $\Delta(z_0;\theta,\rho,\mathbb P_0)\le \epsilon$, it holds that
\[
\mathrm{dist}(\theta,\Theta_{\rm inv})
\le
C\,\big(\Delta(z_0;\theta,\rho,\mathbb P_0)\big)^{1/\kappa}.
\]
\end{assumption}

Assumption~\ref{assm:error_bound_P0} is a Hölder-type \emph{inverse-stability} (error-bound) condition: near the inverse-feasible set $\Theta_{\rm inv}$, the baseline optimality gap acts as a \emph{certificate} of proximity in parameter space. Equivalently, the gap grows at least as a polynomial of degree $\kappa$ as one moves away from $\Theta_{\rm inv}$ within a local neighborhood.

As a consequence, if $\Delta(z_0;\hat\theta,\rho,\mathbb P_0)\to 0$ and $\hat\theta$ eventually enters this neighborhood, then
\[
\mathrm{dist}(\hat\theta,\Theta_{\rm inv})\to 0,
\quad\text{and more generally,}\quad
\Delta(z_0;\hat\theta,\rho,\mathbb P_0)=O_p(r_n)\ \Rightarrow\ \mathrm{dist}(\hat\theta,\Theta_{\rm inv})=O_p\!\left(r_n^{1/\kappa}\right).
\]

This condition is mild and holds generically in many standard stochastic optimization models under basic semi-algebraic or nondegeneracy conditions. For example, $\kappa=1$ corresponds to sharp minima, which frequently occur in linear programs (LPs) or objectives with $\ell_1$-geometry. Conversely, $\kappa=2$ corresponds to quadratic growth, which is standard for quadratic programs (QPs) and strongly convex stochastic optimization problems. In these settings, the inherent stability of the forward optimality conditions guarantees that the optimality gap provides a faithful geometric bound on the parameter recovery error.

\begin{assumption}[Conditional validity of the Wald statistic]
\label{ass:eval_clt} Conditional on
$(\hat\theta,\hat z_1)$, define
\[
G_i
:=
f(z_0;X_i,\hat\theta)-f(\hat z_1;X_i,\hat\theta),
\qquad i\in\mathcal I_{\rm eval}.
\]
The variables $\{G_i\}_{i\in\mathcal I_{\rm eval}}$ are i.i.d. under $\mathbb P_1$, with conditional variance
\[
\sigma^2(\hat\theta,\hat z_1)
:=
\Var_{\mathbb P_1}\!\left(G_i\mid \hat\theta,\hat z_1\right)
\]
satisfying $\sigma(\hat\theta,\hat z_1)=O_p(1)$ and $\sigma(\hat\theta,\hat z_1)^{-1}=O_p(1)$. In addition, the conditional Lindeberg condition holds and $\hat\sigma/\sigma(\hat\theta,\hat z_1)\xrightarrow{p}1$; a sufficient condition is that, for some $\eta>0$,
\[
\sup_{m''}
\E_{\mathbb P_1}\!\left[
\left|G_i-\E_{\mathbb P_1}[G_i\mid\hat\theta,\hat z_1]\right|^{2+\eta}
\mid \hat\theta,\hat z_1
\right]
<\infty
\]
with probability tending to one.
\end{assumption}

\subsection{Proposition, Lemmas and Theorems}

\begin{lemma}[Lipschitz continuity of the deployment gap under OCE]
\label{lem:delta_lipschitz_oce}
Suppose Assumptions~\ref{ass:f_regular}, \ref{ass:f_lipschitz}, and~\ref{ass:risk_class} hold. Then, for any distribution $\mathbb P$ and any fixed $z_0\in\mathcal Z$, the map $\theta \mapsto \Delta(z_0;\theta,\rho,\mathbb P)$ is Lipschitz continuous on $\Theta$. In particular, for all $\theta,\theta'\in\Theta$,
\[
\big|
\Delta(z_0;\theta,\rho,\mathbb P)-\Delta(z_0;\theta',\rho,\mathbb P)
\big|
\le
2L_\ell\,\E_{\mathbb P}[L(X)]\,\|\theta-\theta'\|.
\]
\end{lemma}
\begin{proof}
For $z\in\mathcal Z$, define
\[
R_{\mathbb P}(z,\theta):=\rho^{\mathbb P}\!\big(f(z;X,\theta)\big).
\]
By the OCE representation and the $L_\ell$-Lipschitz continuity of $\ell$, for any $\theta,\theta'\in\Theta$,
\[
\begin{aligned}
|R_{\mathbb P}(z,\theta)-R_{\mathbb P}(z,\theta')|
&\le
\sup_{m\in\mathcal M}
\left|
\E_{\mathbb P}[\ell(f(z;X,\theta)-m)]
-
\E_{\mathbb P}[\ell(f(z;X,\theta')-m)]
\right| \\
&\le
L_\ell\,\E_{\mathbb P}\!\left[|f(z;X,\theta)-f(z;X,\theta')|\right] \\
&\le
L_\ell\,\E_{\mathbb P}[L(X)]\,\|\theta-\theta'\|,
\end{aligned}
\]
where the last step uses Assumption~\ref{ass:f_lipschitz}. Thus, uniformly over $z\in\mathcal Z$, the map $\theta\mapsto R_{\mathbb P}(z,\theta)$ is Lipschitz with constant $L_\ell \E_{\mathbb P}[L(X)]$. Therefore,
\[
\left|
\inf_{z\in\mathcal Z} R_{\mathbb P}(z,\theta)
-
\inf_{z\in\mathcal Z} R_{\mathbb P}(z,\theta')
\right|
\le
L_\ell\,\E_{\mathbb P}[L(X)]\,\|\theta-\theta'\|.
\]
Using
\[
\Delta(z_0;\theta,\rho,\mathbb P)=R_{\mathbb P}(z_0,\theta)-\inf_{z\in\mathcal Z}R_{\mathbb P}(z,\theta),
\]
we obtain
\[
\begin{aligned}
\big|
\Delta(z_0;\theta,\rho,\mathbb P)-\Delta(z_0;\theta',\rho,\mathbb P)
\big|
&\le
|R_{\mathbb P}(z_0,\theta)-R_{\mathbb P}(z_0,\theta')| \\
&\quad+
\left|
\inf_{z\in\mathcal Z}R_{\mathbb P}(z,\theta)-\inf_{z\in\mathcal Z}R_{\mathbb P}(z,\theta')
\right| \\
&\le
2L_\ell\,\E_{\mathbb P}[L(X)]\,\|\theta-\theta'\|.
\end{aligned}
\]
\end{proof}

\begin{lemma}[Uniform convergence of the empirical OCE risk]
\label{lem:uniform_oce_criterion}
Suppose Assumptions~\ref{ass:compact}, \ref{ass:f_regular}, \ref{ass:envelope}, and \ref{ass:risk_class} hold. For each $(z,\theta,m)\in\mathcal Z\times\Theta\times\mathcal M$, define
\[
g_{z,\theta,m}(x)
:=
m+\ell\bigl(f(z,x,\theta)-m\bigr),
\]
and let
\[
\mathcal G
:=
\bigl\{
g_{z,\theta,m}
:\,
(z,\theta,m)\in\mathcal Z\times\Theta\times\mathcal M
\bigr\}.
\]
Then $\mathcal G$ is $\P$-Glivenko--Cantelli. In particular,
\[
\sup_{z\in\mathcal Z,\ \theta\in\Theta,\ m\in\mathcal M}
\left|
\frac1n\sum_{i=1}^n g_{z,\theta,m}(X_i)-\E_\P[g_{z,\theta,m}(X)]
\right|
\xrightarrow{\mathrm{a.s.}}0.
\]
Consequently,
\[
\sup_{z\in\mathcal Z,\ \theta\in\Theta}
\left|
\rho^{\widehat{\mathbb P}_n}\!\big(f(z;X,\theta)\big)
-\rho^{\mathbb P}\!\big(f(z;X,\theta)\big)
\right|
\xrightarrow{\mathrm{a.s.}}0.
\]

\end{lemma}
\begin{proof}
For each $(z,\theta,m)\in\mathcal Z\times\Theta\times\mathcal M$, define
\[
g_{z,\theta,m}(x):=m+\ell(f(z,x,\theta)-m).
\]
We verify that the class
\[
\mathcal G
=
\{g_{z,\theta,m}:(z,\theta,m)\in\mathcal Z\times\Theta\times\mathcal M\}
\]
is $\P$-Glivenko--Cantelli.

First, $\mathcal Z\times\Theta\times\mathcal M$ is compact by Assumptions~\ref{ass:compact} and \ref{ass:risk_class}. Moreover, for almost every $x$, the map
\[
(z,\theta,m)\mapsto g_{z,\theta,m}(x)
=
m+\ell(f(z,x,\theta)-m)
\]
is continuous. Indeed, by Assumption~\ref{ass:f_regular}, $(z,\theta)\mapsto f(z,x,\theta)$ is continuous for almost every $x$, and since $\ell$ is continuous, the composition is continuous in $(z,\theta,m)$.

Next, we show that $\mathcal G$ admits an integrable envelope. Since $\ell$ is Lipschitz with constant $L_\ell$, for any $t\in\R$,
\[
|\ell(t)|\le |\ell(0)|+L_\ell |t|.
\]
Hence, for any $(z,\theta,m)\in\mathcal Z\times\Theta\times\mathcal M$,
\begin{align*}
|g_{z,\theta,m}(X)|
&=
|m+\ell(f(z,X,\theta)-m)| \\
&\le |m| + |\ell(0)| + L_\ell |f(z,X,\theta)-m| \\
&\le |m| + |\ell(0)| + L_\ell |f(z,X,\theta)| + L_\ell |m|.
\end{align*}
Since $\mathcal M$ is compact, $M_{\mathcal M}:=\sup_{m\in\mathcal M}|m|<\infty$. Therefore
\[
|g_{z,\theta,m}(X)|
\le
(1+L_\ell)M_{\mathcal M}+|\ell(0)|+L_\ell F(X)
\quad\text{a.s.},
\]
where we used Assumption~\ref{ass:envelope}. Thus $\mathcal G$ is dominated by the measurable envelope
\[
G(X):=(1+L_\ell)M_{\mathcal M}+|\ell(0)|+L_\ell F(X),
\]
and $\E[G(X)]<\infty$.

Therefore, $\mathcal G$ is a compactly indexed family of measurable functions, continuous in the index for almost every $x$, with integrable envelope. By the standard parametric Glivenko--Cantelli theorem, $\mathcal G$ is $\P$-Glivenko--Cantelli. Hence
\[
\sup_{z\in\mathcal Z,\ \theta\in\Theta,\ m\in\mathcal M}
\left|
\frac1n\sum_{i=1}^n g_{z,\theta,m}(X_i)-\E_\P[g_{z,\theta,m}(X)]
\right|
\xrightarrow{\mathrm{a.s.}}0.
\]
This proves the claim.
\end{proof}

\begin{lemma}[Uniform rate of convergence of the empirical OCE risk]
\label{lem:rate_oce_criterion}
Suppose Assumptions in Lemma \ref{lem:uniform_oce_criterion} and Assumption~\ref{ass:f_lipschitz} hold, then
\[
\sup_{z\in\mathcal Z,\ \theta\in\Theta,\ m\in\mathcal M}
\left|
\frac1n\sum_{i=1}^n g_{z,\theta,m}(X_i)-\E_\P[g_{z,\theta,m}(X)]
\right|
=
O_p(n^{-1/2}).
\]
Consequently,
\[
\sup_{z\in\mathcal Z,\ \theta\in\Theta}
\left|
\rho^{\widehat{\mathbb P}_n}\!\big(f(z;X,\theta)\big)
-\rho^{\mathbb P}\!\big(f(z;X,\theta)\big)
\right|
=O_p(n^{-1/2}).
\]
\end{lemma}
\begin{proof}
As in the proof of Lemma~\ref{lem:uniform_oce_criterion}, the class
\[
\mathcal G
=
\{g_{z,\theta,m}:(z,\theta,m)\in\mathcal Z\times\Theta\times\mathcal M\}
\]
admits the envelope
\[
G(X):=(1+L_\ell)M_{\mathcal M}+|\ell(0)|+L_\ell F(X).
\]
Under the assumption $\E[F(X)^2]<\infty$, we also have $\E[G(X)^2]<\infty$.

Next, we verify a Lipschitz modulus in the indexing variables. For any $(z,\theta,m),(z',\theta',m')\in\mathcal Z\times\Theta\times\mathcal M$,
\begin{align*}
|g_{z,\theta,m}(X)-g_{z',\theta',m'}(X)|
&=
\Big|
m+\ell(f(z,X,\theta)-m)
-
m'-\ell(f(z',X,\theta')-m')
\Big| \\
&\le |m-m'|
+L_\ell\big(|f(z,X,\theta)-f(z',X,\theta')|+|m-m'|\big) \\
&\le (1+L_\ell)|m-m'|
+L_\ell L(X)\bigl(\|z-z'\|+\|\theta-\theta'\|\bigr),
\end{align*}
where we used Assumption~\ref{ass:f_lipschitz}.

Thus $\mathcal G$ is a finite-dimensional Lipschitz-parametric class with square-integrable envelope and square-integrable random Lipschitz modulus. Standard maximal inequalities for such parametric empirical process classes yield
\[
\sup_{g\in\mathcal G}|(\P_n-\P)g|=O_p(n^{-1/2}),
\]
that is,
\[
\sup_{z\in\mathcal Z,\ \theta\in\Theta,\ m\in\mathcal M}
\left|
\frac1n\sum_{i=1}^n g_{z,\theta,m}(X_i)-\E_{\P}[g_{z,\theta,m}(X)]
\right|
=
O_p(n^{-1/2}).
\]
This proves the result.
\end{proof}

\subsubsection{Consistency and Rates for Term \textup{($i$)} in \eqref{eq:delta_decomp}}
We first establish that the estimator $\hat\theta$ is consistent for the inverse-feasible set $\Theta_{\rm inv}$ in the following lemma.

\begin{lemma}
\label{lem:IO_set_consistency}

Assume the same assumptions as in Lemma \ref{lem:uniform_oce_criterion}. For $\hat \theta$ defined in \eqref{eq:iop}, it holds that $\mathrm{dist}(\hat\theta,\Theta_{\rm inv})\xrightarrow{p}0$, where $\mathrm{dist}(\theta,S):=\inf_{\vartheta\in S}\|\theta-\vartheta\|_2$.
\end{lemma}

\begin{proof}
Recall the definition of $\hat\theta$:
\[
\hat\theta \in \argmin_{\theta\in\Theta}
\Big\{\rho^{\widehat{\mathbb P}_0}\big[f(z_0;X,\theta)\big] - \inf_{z\in\mathcal Z} \rho^{\widehat{\mathbb P}_0}\big[f(z;X,\theta)\big]
\Big\}.
\]
We know that the population minimizer is $\theta^\star\in\argmin_{\theta\in\Theta}\Big\{\rho^{\mathbb P_0}\big[f(z_0;X,\theta)\big]-\inf_{z\in\mathcal Z}\rho^{\mathbb P_0}\big[f(z;X,\theta)\big]\Big\}$.

We first observe the following inequality: for any $\theta\in\Theta$,
\begin{equation}
\label{eq:inf_lip}
\Big |\inf_{z\in\mathcal Z} \rho^{\widehat{\mathbb P}_0}\big[f(z;X,\theta)\big] - \inf_{z\in\mathcal Z} \rho^{{\mathbb P}_0}\big[f(z;X,\theta)\big] \Big|
\leq \sup_{z \in \mathcal Z} \Big|\rho^{\widehat{\mathbb P}_0}\big[f(z;X,\theta)\big] - \rho^{{\mathbb P}_0}\big[f(z;X,\theta)\big] \Big|.
\end{equation}
This can be seen by letting $a(z):=\rho^{\widehat{\mathbb P}_0}[f(z;X,\theta)]$ and $b(z):=\rho^{\mathbb P_0}[f(z;X,\theta)]$, and defining $\Delta:=\sup_{z\in\mathcal Z}|a(z)-b(z)|$. Then for every $z\in\mathcal Z$ we have $a(z)\le b(z)+\Delta$, hence taking $\inf_{z}$ yields
\[
\inf_{z} a(z)\le \inf_{z} b(z)+\Delta
\quad\Longrightarrow\quad
\inf_{z} a(z)-\inf_{z} b(z)\le \Delta.
\]
By swapping $a$ and $b$ we also obtain $\inf_{z} b(z)-\inf_{z} a(z)\le \Delta$, i.e., $\inf_{z} a(z)-\inf_{z} b(z)\ge -\Delta$. Combining the two inequalities gives
\[
\bigl|\inf_{z} a(z)-\inf_{z} b(z)\bigr|\le \Delta
=\sup_{z\in\mathcal Z}|a(z)-b(z)|,
\]
which proves \eqref{eq:inf_lip}.

Define the population and empirical inverse objectives (as functions of $\theta$) by
\[
\Psi(\theta)
:=\rho^{\mathbb P_0}\big[f(z_0;X,\theta)\big]-\inf_{z\in\mathcal Z}\rho^{\mathbb P_0}\big[f(z;X,\theta)\big],
\qquad
\widehat\Psi_n(\theta)
:=\rho^{\widehat{\mathbb P}_0}\big[f(z_0;X,\theta)\big]-\inf_{z\in\mathcal Z}\rho^{\widehat{\mathbb P}_0}\big[f(z;X,\theta)\big].
\]
Then for any $\theta\in\Theta$,
\begin{align}
\big|\widehat\Psi_n(\theta)-\Psi(\theta)\big|
&\le
\Big|\rho^{\widehat{\mathbb P}_0}\big[f(z_0;X,\theta)\big]-\rho^{\mathbb P_0}\big[f(z_0;X,\theta)\big]\Big|
+ \Big|\inf_{z}\rho^{\widehat{\mathbb P}_0}\big[f(z;X,\theta)\big]-\inf_{z}\rho^{\mathbb P_0}\big[f(z;X,\theta)\big]\Big| \notag\\
&\le
\Big|\rho^{\widehat{\mathbb P}_0}\big[f(z_0;X,\theta)\big]-\rho^{\mathbb P_0}\big[f(z_0;X,\theta)\big]\Big|
+ \sup_{z\in\mathcal Z}\Big|\rho^{\widehat{\mathbb P}_0}\big[f(z;X,\theta)\big]-\rho^{\mathbb P_0}\big[f(z;X,\theta)\big]\Big| \notag\\
&\le
2\sup_{z\in\mathcal Z}\Big|\rho^{\widehat{\mathbb P}_0}\big[f(z;X,\theta)\big]-\rho^{\mathbb P_0}\big[f(z;X,\theta)\big]\Big|. \label{eq:psi_dev}
\end{align}
By Lemma \ref{lem:uniform_oce_criterion}, the right-hand side converges to $0$ almost surely (and hence in probability). Combining this with \eqref{eq:psi_dev} yields uniform convergence of
\begin{equation}
\label{eq:unifPsi}
\sup_{\theta\in\Theta}\big|\widehat\Psi_n(\theta)-\Psi(\theta)\big|\xrightarrow{p}0.
\end{equation}

Finally, since $\hat\theta$ minimizes $\widehat\Psi_n$ over $\Theta$, we have
\[
\Psi(\hat\theta)-\Psi(\theta^\star)
\le
\big(\Psi(\hat\theta)-\widehat\Psi_n(\hat\theta)\big)
+\big(\widehat\Psi_n(\hat\theta)-\widehat\Psi_n(\theta^\star)\big)
+\big(\widehat\Psi_n(\theta^\star)-\Psi(\theta^\star)\big)
\le
2\sup_{\theta\in\Theta}\big|\widehat\Psi_n(\theta)-\Psi(\theta)\big|.
\]
By \eqref{eq:unifPsi}, the right-hand side converges to $0$ in probability, hence
\[
\Psi(\hat\theta)-\inf_{\theta\in\Theta}\Psi(\theta)\xrightarrow{p}0.
\]
In particular, any limit point of $\hat\theta$ must lie in $\Theta_{\rm inv }= \argmin_{\theta\in\Theta}\Psi(\theta)$, and therefore
\[
\dist\!\Big(\hat\theta,\Theta_{\rm inv }\Big)\xrightarrow{p}0,
\]
which proves the desired set consistency.
\end{proof}

\begin{lemma}[Consistency of \textup{($i$)} in \eqref{eq:delta_decomp}]
\label{lem:termI_consistency}
Suppose the assumptions in Lemma~\ref{lem:IO_set_consistency} hold. Furthermore, assume that the envelope condition (Assumption \ref{ass:envelope}) and the Lipschitz integrability condition (Assumption \ref{ass:f_lipschitz}) hold under the target distribution $\mathbb{P}_1$, namely $\mathbb{E}_{\mathbb{P}_1}[F(X)^2] < \infty$ and $\mathbb{E}_{\mathbb{P}_1}[L(X)^2] < \infty$.

It holds that
\begin{equation}
\label{eq:termI_bound_final}
\big|\Delta(z_0;\hat \theta,\rho,\mathbb P_1)-\Delta(z_0;\theta^\star,\rho,\mathbb P_1)\big|
\le
2L_\ell \mathbb{E}_{\mathbb{P}_1}[L(X)]\,\mathrm{dist}(\hat\theta,\Theta_{\rm inv})
+\delta_{\rm inv}.
\end{equation}
Consequently,
\[
\big|\Delta(z_0;\hat \theta,\rho,\mathbb P_1)-\Delta(z_0;\theta^\star,\rho,\mathbb P_1)\big|
\le\delta_{\rm inv}+o_p(1),
\]
and if $\delta_{\rm inv}=0$ then $\big|\Delta(z_0;\hat \theta,\rho,\mathbb P_1)-\Delta(z_0;\theta^\star,\rho,\mathbb P_1)\big|\xrightarrow{p}0$.
\end{lemma}

\begin{proof}

Define $R_1(z,\theta):=\rho^{\mathbb P_1}(f(z;X,\theta))$. By Assumption~\ref{ass:f_lipschitz} and \ref{ass:risk_class}, for any $z\in\mathcal Z$ and $\theta,\theta'\in\Theta$,
\[
|R_1(z,\theta)-R_1(z,\theta')|
\le
L_\ell\,\mathbb E_{\mathbb P_1}\!\left[|f(z;X,\theta)-f(z;X,\theta')|\right]
\le
L_\ell \E_{\P_1}[L(X)] \|\theta-\theta'\|_2.
\]
Also we have
\[
\Big|\min_{z\in\mathcal Z}R_1(z,\theta)-\min_{z\in\mathcal Z}R_1(z,\theta')\Big|
\le
\sup_{z\in\mathcal Z}|R_1(z,\theta)-R_1(z,\theta')|
\le
L_\ell \E_{\P_1}[L(X)] \|\theta-\theta'\|_2.
\]
Thus
\begin{align*}
    |R_1(z_0,\theta)-\min_{z\in\mathcal Z}R_1(z,\theta) - (R_1(z_0,\theta')-\min_{z\in\mathcal Z}R_1(z,\theta'))| \le
    2L_\ell \E_{\P_1}[L(X)] \|\theta-\theta'\|_2.
\end{align*}
Notice that
\[
\Delta(z_0;\theta,\rho,\P_1) = R_1(z_0,\theta)-\min_{z\in\mathcal Z}R_1(z,\theta),
\]
then let $\theta^\pi\in\argmin_{\vartheta\in\Theta_{\rm inv}}\|\hat\theta-\vartheta\|_2$ denote the projection of $\hat \theta$ on $\Theta_{\rm inv}$. We have
\begin{align*}
    |\Delta(z_0;\hat \theta,\rho,\P_1)-\Delta(z_0;\theta^\star,\rho,\P_1)|
&\le |\Delta(z_0;\hat \theta,\rho,\P_1)-\Delta(z_0;\theta^\pi ,\rho,\P_1)|+|\Delta(z_0;\theta^\pi,\rho,\P_1)-\Delta(z_0;\theta^\star,\rho,\P_1)|\\
&\le 2L_\ell \E_{\P_1}[L(X)]\,\mathrm{dist}(\hat\theta,\Theta_{\rm inv})+\delta_{\rm inv},
\end{align*}
where the last step uses Lipschitz continuity and the risk-specific definition of $\delta_{\rm inv}$ above. The stated consistency follows from $\mathrm{dist}(\hat\theta,\Theta_{\rm inv})=o_p(1)$ in Lemma \ref{lem:IO_set_consistency}.

\end{proof}

\begin{lemma}[Rate of convergence for \textup{($i$)}]
\label{lem:termI_rate}
Suppose the assumptions of Lemma~\ref{lem:termI_consistency} hold. Let $\varepsilon_n$ denote the uniform convergence rate of the empirical OCE risk over $\mathcal{Z} \times \Theta$ such that $\varepsilon_n = O_p(n^{-1/2})$ (as established in Lemma \ref{lem:rate_oce_criterion}). Under the $\kappa$-degree local inverse stability condition (Assumption \ref{assm:error_bound_P0}), it holds that:
\begin{equation}
\label{eq:theta_rate}
\mathrm{dist}(\hat\theta,\Theta_{\rm inv})=O_p\!\big(\varepsilon_n^{1/\kappa}\big).
\end{equation}
Furthermore, the deployment optimality gap satisfies:
\begin{equation}
\label{eq:termI_rate}
\Big|\Delta(z_0;\hat \theta,\rho,\mathbb P_1)-\Delta(z_0;\theta^\star,\rho,\mathbb P_1)\Big|
\le \delta_{\rm inv}+O_p\!\big(\varepsilon_n^{1/\kappa}\big).
\end{equation}
In particular, since $\varepsilon_n \asymp n^{-1/2}$, we have $\mathrm{dist}(\hat\theta,\Theta_{\rm inv})=O_p(n^{-1/(2\kappa)})$. If $\delta_{\rm inv}=0$, the gap strictly converges at the rate $O_p(n^{-1/(2\kappa)})$.
\end{lemma}

\begin{proof}
Let $\Psi(\theta) := \Delta(z_0; \theta, \rho, \mathbb{P}_0)$ and $\widehat\Psi_n(\theta) := \Delta(z_0; \theta, \rho, \widehat{\mathbb{P}}_0)$ denote the population and empirical inverse objectives under the baseline distribution.

From the decomposition in the proof of Lemma \ref{lem:IO_set_consistency}, we know that for any $\theta \in \Theta$:
\[
\big|\widehat\Psi_n(\theta)-\Psi(\theta)\big| \le 2\sup_{z\in\mathcal Z}\Big|\rho^{\widehat{\mathbb P}_0}\big[f(z;X,\theta)\big]-\rho^{\mathbb P_0}\big[f(z;X,\theta)\big]\Big|.
\]
By Lemma \ref{lem:rate_oce_criterion}, the empirical OCE risk converges uniformly over $\mathcal{Z} \times \Theta$ at the rate $\varepsilon_n = O_p(n^{-1/2})$. Therefore, the inverse objective converges uniformly at the same rate:
\begin{equation}
\label{eq:psi_unif_rate}
\sup_{\theta \in \Theta} \big|\widehat\Psi_n(\theta)-\Psi(\theta)\big| = O_p(\varepsilon_n).
\end{equation}

By definition, $\hat\theta$ minimizes the empirical objective $\widehat\Psi_n(\theta)$ over $\Theta$. Let $\theta^\star \in \Theta_{\rm inv}$ be a true minimizer of the population objective, which implies $\Psi(\theta^\star) = 0$. We bound the population suboptimality of $\hat\theta$ using the standard empirical risk minimization decomposition:
\begin{align*}
\Psi(\hat\theta) &= \Psi(\hat\theta) - \Psi(\theta^\star) \\
&\le \big(\Psi(\hat\theta) - \widehat\Psi_n(\hat\theta)\big) + \big(\widehat\Psi_n(\hat\theta) - \widehat\Psi_n(\theta^\star)\big) + \big(\widehat\Psi_n(\theta^\star) - \Psi(\theta^\star)\big).
\end{align*}
Because $\hat\theta$ is the empirical minimizer, the middle term satisfies $\widehat\Psi_n(\hat\theta) - \widehat\Psi_n(\theta^\star) \le 0$. Bounding the remaining terms by the uniform supremum yields:
\[
\Psi(\hat\theta) \le 2 \sup_{\theta \in \Theta} \big|\widehat\Psi_n(\theta)-\Psi(\theta)\big| = O_p(\varepsilon_n).
\]

Since $\Psi(\hat\theta) \xrightarrow{p} 0$, $\hat\theta$ eventually falls within the local neighborhood $\epsilon$ specified by Assumption \ref{assm:error_bound_P0}. Applying the $\kappa$-degree inverse stability bound yields:
\[
\mathrm{dist}(\hat\theta, \Theta_{\rm inv}) \le C \big(\Psi(\hat\theta)\big)^{1/\kappa} = O_p\big(\varepsilon_n^{1/\kappa}\big).
\]
This proves \eqref{eq:theta_rate}.

To bound the deployment gap, we apply the result from Lemma \ref{lem:termI_consistency}, which establishes the Lipschitz continuity of the gap:
\[
\big|\Delta(z_0;\hat \theta,\rho,\mathbb P_1)-\Delta(z_0;\theta^\star,\rho,\mathbb P_1)\big| \le 2 L_\ell \mathbb{E}_{\mathbb{P}_1}[L(X)] \,\mathrm{dist}(\hat\theta,\Theta_{\rm inv}) + \delta_{\rm inv}.
\]
Substituting the convergence rate of the distance into this bound gives:
\[
\big|\Delta(z_0;\hat \theta,\rho,\mathbb P_1)-\Delta(z_0;\theta^\star,\rho,\mathbb P_1)\big| \le \delta_{\rm inv} + O_p\big(\varepsilon_n^{1/\kappa}\big).
\]
Since $\varepsilon_n = O_p(n^{-1/2})$, this yields an overall rate of $O_p(n^{-1/(2\kappa)})$ for the parameter estimation error component, completing the proof.
\end{proof}
\subsubsection{Consistency and Rates for Term \textup{($ii$)} in \eqref{eq:delta_decomp}}

\begin{lemma}[Consistency and rate for \textup{($ii$)}]
\label{lem:termII_rate}
Suppose the assumptions of Lemma~\ref{lem:termI_consistency} hold, specifically that the envelope $F(X)$ and the Lipschitz modulus $L(X)$ are square-integrable under the target distribution $\mathbb{P}_1$.
\begin{equation}
\label{eq:termII_bound}
|\textup{($ii$)}|
=\Big|
\min_{z\in\mathcal Z}
\rho^{\mathbb P_1}\!\big(f(z;X,\hat\theta)\big)
-
\rho^{\mathbb P_1}\!\big(f(\hat z_1;X,\hat\theta)\big)
\Big| \leq 2\,\sup_{\theta\in\Theta}\sup_{z\in\mathcal Z}\big|\rho^{\widehat{\mathbb P}_1'}\!\big(f(z;X,\theta)\big) - \rho^{\mathbb P_1}\!\big(f(z;X,\theta)\big)\big|
=O_p(\varepsilon_{m'}).
\end{equation}
In particular, if $\varepsilon_{m'}\to 0$ then $\textup{($ii$)}\xrightarrow{p}0$, and if $\varepsilon_{m'}\asymp (m')^{-1/2}$ then $|\textup{($ii$)}|=O_p((m')^{-1/2})$.

\end{lemma}

\begin{proof}
Define
\[
R_1(z,\theta):=\rho^{\mathbb P_1}\!\big(f(z;X,\theta)\big),
\qquad
\widehat R_1'(z,\theta):=\rho^{\widehat{\mathbb P}_1'}\!\big(f(z;X,\theta)\big).
\]
Fix $\hat\theta$ and let $z_1(\hat\theta)\in\arg\min_{z\in\mathcal Z}R_1(z,\hat\theta)$. By adding and subtracting $\widehat R_1'$,
\begin{align*}
R_1(\hat z_1,\hat\theta)-R_1(z_1(\hat\theta),\hat\theta)
&=
\underbrace{\Big(R_1(\hat z_1,\hat\theta)-\widehat R_1'(\hat z_1,\hat\theta)\Big)}_{(a)}
+\underbrace{\Big(\widehat R_1'(\hat z_1,\hat\theta)-\widehat R_1'(z_1(\hat\theta),\hat\theta)\Big)}_{(b)}
+\underbrace{\Big(\widehat R_1'(z_1(\hat\theta),\hat\theta)-R_1(z_1(\hat\theta),\hat\theta)\Big)}_{(c)}.
\end{align*}
By definition of $\hat z_1$, $(b)\le 0$. Hence,
\[
R_1(\hat z_1,\hat\theta)-\min_{z\in\mathcal Z}R_1(z,\hat\theta)
\le |(a)|+|(c)|
\le 2\sup_{z\in\mathcal Z}\big|R_1(z,\hat\theta)-\widehat R_1'(z,\hat\theta)\big|
\le 2\sup_{\theta\in\Theta}\sup_{z\in\mathcal Z}\big|R_1(z,\theta)-\widehat R_1'(z,\theta)\big|,
\]
which proves \eqref{eq:termII_bound}.
\end{proof}

\subsubsection{Consistency and Rates for Term \textup{($iii$)} in \eqref{eq:delta_decomp}}

\begin{lemma}[Consistency and rate for \textup{($iii$)}]
\label{lem:termIII_rate}
Suppose the assumptions of Lemma~\ref{lem:termI_consistency} hold, specifically the square-integrability of the envelope and Lipschitz modulus under $\mathbb{P}_1$. Let $\varepsilon_{m''}$ denote the uniform convergence rate of the empirical OCE risk over the evaluation sample. It holds that:
\begin{equation}
\label{eq:termIII_bound}
|\textup{($iii$)}|
\le
2\sup_{\theta\in\Theta}\sup_{z\in\mathcal Z}
\Big|
\rho^{\widehat{\mathbb P}_1''}\!\big(f(z;X,\theta)\big)
-
\rho^{\mathbb P_1}\!\big(f(z;X,\theta)\big)
\Big|
=O_p(\varepsilon_{m''}).
\end{equation}
In particular, if $\varepsilon_{m''}\to 0$ then $\textup{($iii$)}\xrightarrow{p}0$, and if $\varepsilon_{m''}\asymp (m'')^{-1/2}$ then $|\textup{($iii$)}|=O_p((m'')^{-1/2})$.
\end{lemma}

\begin{proof}
Define the population and evaluation-split risks
\[
R_1(z,\theta):=\rho^{\mathbb P_1}\!\big(f(z;X,\theta)\big),
\qquad
\widehat R_1''(z,\theta):=\rho^{\widehat{\mathbb P}_1''}\!\big(f(z;X,\theta)\big).
\]
By the triangle inequality,
\begin{align*}
|\textup{($iii$)}|
&=
\Big|\big(\widehat R_1''(z_0,\hat\theta)-R_1(z_0,\hat\theta)\big)
-
\big(\widehat R_1''(\hat z_1,\hat\theta)-R_1(\hat z_1,\hat\theta)\big)\Big|\\
&\le
\big|\widehat R_1''(z_0,\hat\theta)-R_1(z_0,\hat\theta)\big|
+
\big|\widehat R_1''(\hat z_1,\hat\theta)-R_1(\hat z_1,\hat\theta)\big|\\
&\le
2\sup_{\theta\in\Theta}\sup_{z\in\mathcal Z}\big|\widehat R_1''(z,\theta)-R_1(z,\theta)\big|,
\end{align*}
which proves \eqref{eq:termIII_bound}. The stated consistency and rate follow from the assumed uniform concentration bound for $\widehat R_1''$.
\end{proof}

\begin{proof}[Proof of Theorem~\ref{thm:delta_hat_rate}]
Apply Lemmas~\ref{lem:termI_consistency}, \ref{lem:termII_rate}, and~\ref{lem:termIII_rate} to the decomposition in \eqref{eq:delta_decomp}. Their sum is bounded by $\delta_{\rm inv}+o_p(1)$. Under Assumption~\ref{assm:error_bound_P0}, Lemma~\ref{lem:termI_rate} sharpens the preference term to $\delta_{\rm inv}+O_p(n^{-1/(2\kappa)})$, while Lemmas~\ref{lem:termII_rate} and~\ref{lem:termIII_rate} give the $(m')^{-1/2}$ and $(m'')^{-1/2}$ terms. The stated bound follows by the triangle inequality.
\end{proof}

\begin{proof}[Proof of Theorem~\ref{thm:power_rate}]
Conditional on $(\hat\theta,\hat z_1)$, let
\[
\mu(\hat\theta,\hat z_1)
:=\mathbb E_{\mathbb P_1}
\!\left[f(z_0;X,\hat\theta)-f(\hat z_1;X,\hat\theta)
\mid\hat\theta,\hat z_1\right].
\]
Assumption~\ref{ass:eval_clt} and Slutsky's theorem imply
\[
\frac{\sqrt{m''}\{\widehat\Delta-\mu(\hat\theta,\hat z_1)\}}
{\hat\sigma}
\;\Rightarrow\;N(0,1).
\]
The benchmark term in \eqref{eq:delta_decomp} is nonpositive. Therefore, when $\delta_{\rm inv}=0$,
\[
\mu(\hat\theta,\hat z_1)-\Delta^\star
\le O_p(a_n).
\]
If $\Delta^\star\le\tau$ and $\sqrt{m''}a_n\to0$, the positive centering error is $o_p(1)$ on the studentized scale, so the one-sided rejection probability is at most $\alpha+o(1)$. On a fixed-margin null, consistency alone makes $T_{\rm Wald}\to-\infty$, which proves the separated-null assertion. If $\delta_{\rm inv}>0$, the same argument applies after enlarging the possible positive centering error by $\delta_{\rm inv}$.

Under $\Delta^\star\ge\tau+\delta_{\rm inv}+\eta$, the consistency and rate bounds imply $\mu(\hat\theta,\hat z_1)-\tau\ge\eta+o_p(1)$; hence $T_{\rm Wald}\to\infty$ and the power tends to one. Finally, replacing $\tau$ by $\tau+\bar\delta_{\rm inv}$ offsets the largest admissible positive identification error, yielding the sensitivity-adjusted level statement under the same relative-rate and CLT conditions.
\end{proof}

\section{Details of Experimental Configuration and Additional Results}
\label{app:experiment}

\subsection{Implementation Configuration}
\label{app:config}

We implement all pipelines in PyTorch using differentiable optimization layers from \texttt{cvxpylayers}, with forward optimization solved by ECOS. For inverse optimization, we estimate the latent preference parameter via multi-start gradient training with a suboptimality loss. Unless otherwise specified, we train for 300 epochs with learning rate $5\times 10^{-3}$ and 10 random initializations, using 3000 contexts for preference estimation, 2000 contexts for challenger construction, and 2000 contexts for evaluation. In the synthetic experiments, we set the decision-adequacy tolerance to $\tau=0$, corresponding to exact optimality. In the semi-synthetic experiments, we use a relative tolerance $\tau=0.1$, under which the deployed decision is considered adequate if its estimated risk is within $10\%$ of the optimal risk; the districting study uses $\tau=0.2$ at level $\alpha=0.01$, as recorded in Appendix~\ref{app:atlanta}.

\paragraph{Compute resources.}
All experiments were implemented in PyTorch with \texttt{cvxpylayers} for differentiable optimization and ECOS for conic optimization. Experiments were run on a local Apple M4 machine with 16~GB of memory and a Windows PC equipped with an NVIDIA RTX 5070 Ti GPU. The synthetic and Fashion-MNIST experiments each required less than one minute per repeated run, while the Atlanta districting experiment required several seconds per sliding-window evaluation.

\subsection{Details of Synthetic Experiments}
\label{app:synthetic_experiment}

\paragraph{Baseline methods.}
Across all settings, we compare \texttt{RADAR} against three alternative tests that target different notions of distributional or operational change.

\emph{X-Mean} is a two-sample Wald test for equality of the context means. Given independent baseline samples $X_1,\ldots,X_n\sim\mathbb P_0$ and target samples $X'_1,\ldots,X'_m\sim\mathbb P_1$, it tests
\[
H_0^{\mathrm{X\text{-}Mean} }:
\E_{\mathbb P_0}[X]
=
\E_{\mathbb P_1}[X]
\qquad\text{vs.}\qquad
H_1^{\mathrm{X\text{-}Mean} }:
\E_{\mathbb P_0}[X]
\neq
\E_{\mathbb P_1}[X].
\]
For scalar contexts, the Wald statistic is
\[
T_{\mathrm{X\text{-}Mean}}
=
\frac{\bar X'-\bar X}
{\sqrt{\hat\sigma_1^2/m+\hat\sigma_0^2/n}},
\]
where $\bar X,\bar X'$ and $\hat\sigma_0^2,\hat\sigma_1^2$ denote the sample means and variances in the baseline and target domains, respectively. For multivariate contexts, the same test is applied to the scalar context summary used in the corresponding experiment.

\emph{X-Distr} is a two-sample test for equality of the full context distributions,
\[
H_0^{\mathrm{X\text{-}Distr}}:
\mathbb P_0=\mathbb P_1
\qquad\text{vs.}\qquad
H_1^{\mathrm{X\text{-}Distr}}:
\mathbb P_0\neq\mathbb P_1.
\]
For multivariate contexts, we project the observations onto multiple random one-dimensional directions and compute a two-sample Kolmogorov--Smirnov statistic along each projection. The global statistic is the maximum projected discrepancy,
\[
T_{\mathrm{X\text{-}Distr}}
=
\max_{r=1,\ldots,R}
\mathrm{KS}_r,
\]
and its $p$-value is obtained by permutation calibration. Thus, X-Distr is sensitive to general changes in the context distribution beyond first-moment shifts.

\emph{Risk-Value} is a decision-related baseline~\citep{podkopaev2022tracking} that tests for changes in the realized risk of the deployed decision $z_0$. Using the plug-in preference estimate $\hat\theta$ and the same risk functional as the forward problem, it compares the baseline and target risks
\[
\E_{\mathbb P_0}\!\left[f(z_0;X,\hat\theta)\right]
\qquad\text{and}\qquad
\E_{\mathbb P_1}\!\left[f(z_0;X,\hat\theta)\right].
\]
For expectation risk, Risk-Value reduces to a two-sample mean test on the realized objective values; for CVaR, it compares the corresponding empirical upper-tail averages.

In contrast to these baselines, \texttt{RADAR} tests whether the incumbent decision is suboptimal relative to the target-domain optimum by more than the tolerance $\tau$. We report empirical rejection rates over repeated trials, corresponding to Type-I error under $H_0$ and power under $H_1$.


\paragraph{Evaluation metrics.}
The primary metric is the empirical rejection rate over repeated simulations. For shifts under which the null remains valid, the rejection rate should remain close to the nominal level $\alpha$. For shifts violating the null, a powerful test should reject with increasing probability as the shift magnitude $\delta$ grows.

\paragraph{Optimization problem setup.}
We consider two canonical settings: an LP with CVaR risk and a QP with expectation risk. Let $X\in\mathbb R^{n\times d}$ denote a random context matrix, where the rows represent items, scenarios, or units, and each row contains a $d$-dimensional feature vector. Each row of $X$ lies on the simplex. The latent parameter $\theta\in\Delta^{d-1}$ and decision variable $z\in\Delta^{n-1}$ also lie on the simplex. Thus, $X\theta\in\mathbb R^n$ can be interpreted as a context-dependent score or cost vector induced by the latent preference parameter, while $z$ allocates resources, weights, or attention across feasible actions. This formulation covers resource allocation, portfolio-style allocation, and distribution-dependent assignment problems.

For the LP + CVaR setting, we solve
\begin{align}
\arg\min_{z \in \Delta^{n-1}}
&\ \CVaR_{\alpha}\!\left( \langle X\theta, z \rangle \right)
\label{eq:lp-cvar-pop}
\\
\text{s.t.}\quad
&Az \le b .
\nonumber
\end{align}
The CVaR objective emphasizes tail risk, making the optimal decision sensitive to adverse realizations of the random context.

For the QP + expectation setting, we solve
\begin{align}
\arg\min_{z \in \Delta^{n-1}}
&\ \E\!\left[
\frac{1}{2} z^\top Q z + \langle X\theta, z \rangle
\right]
\label{eq:qp-exp-pop}
\\
\text{s.t.}\quad
&Az \le b .
\nonumber
\end{align}
Here, the quadratic term introduces curvature and regularization in the decision space, while the stochastic linear term captures average performance under the context distribution. Together, these two settings represent complementary decision regimes: LP + CVaR is sensitive to tail behavior, whereas QP + expectation is driven primarily by average performance.

\paragraph{Data generation and distributional shifts.}
We generate the random context matrix $X\in\mathbb R^{n\times d}$ row-wise from Dirichlet distributions,
\[
X_i \sim \Dir(\Alpha_i), \qquad i=1,\dots,n,
\]
where $\Alpha\in\mathbb R^{n\times d}_{+}$ is the matrix of Dirichlet parameters and $\Alpha_i$ denotes its $i$th row. Thus, each row of $X$ lies on the simplex, and $\Alpha$ controls the row-wise context distributions. We treat $\Alpha$ as the population distribution parameter and introduce distributional shifts by perturbing it to a shifted matrix $\Alpha'$.

Table~\ref{tab:shift-summary} summarizes the structured shifts used in the experiments. Each shift is indexed by a magnitude $\delta$, which controls the size of the perturbation from $\Alpha$ to $\Alpha'$.

\begin{table}[!tbh]
\centering
\begin{tabular}{lll}
\toprule
\textbf{Shift type} & \textbf{Formulation} & \textbf{Description} \\
\midrule
Global scaling 
& $\Alpha' = (1+\delta)\Alpha$ 
& Scale all rows uniformly \\

Row scaling 
& $\Alpha'_{i,:} = (1+\delta)\Alpha_{i,:}$ 
& Scale one selected row \\

Orthogonal shift 
& $\Alpha'_{i,:} = \Alpha_{i,:} + \delta v,\ \langle v,\theta\rangle=0$ 
& Perturb along a direction orthogonal to $\theta$ \\

Localized mean 
& $\Alpha'_{i,j} = \Alpha_{i,j} + \delta$ 
& Increase one coordinate in a selected row \\

Row redistribution 
& $\Alpha'_{i,j_1} = \Alpha_{i,j_1} + \delta,\ \Alpha'_{i,j_2} = \Alpha_{i,j_2} - \delta$ 
& Move mass within a selected row \\

Tail-sensitive 
& selective perturbation on tail rows 
& Modify rows in the worst $\alpha$-tail under CVaR \\
\bottomrule
\end{tabular}
\caption{Structured distributional shifts considered in the synthetic experiments.}
\label{tab:shift-summary}
\end{table}

\subsection{Details of Semi-Synthetic and Real-Data Experiments}
\label{app:real_data}
\subsubsection{Multi-Product Newsvendor Capacity Allocation with Fashion-MNIST}
\label{app:fashion_mnist}

We provide additional details for the capacity-constrained multi-product newsvendor task introduced in Section~\ref{sec:experiments}. Let $\mathcal G:=\{\mathrm{cloth},\mathrm{foot},\mathrm{acc}\}$ denote the three coarse groups obtained from the ten Fashion-MNIST classes: \emph{clothing} $\{0,1,2,3,4,6\}$, \emph{footwear} $\{5,7,9\}$, and \emph{accessory} $\{8\}$. Each image is a context $x$ associated with one unit of demand, and $\operatorname{grp}(x)\in\mathcal G$ denotes its coarse group. Before future demand contexts are realized, the decision-maker chooses
\[
\begin{aligned}
z&=(z_{\mathrm{cloth}},z_{\mathrm{foot}},z_{\mathrm{acc}})\in\Delta^2,\\
\Delta^2&:=\left\{z\in\mathbb R_+^3:
              \sum_{g\in\mathcal G}z_g=1\right\}.
\end{aligned}
\]
Let $y_g(x):=\mathbf 1\{\operatorname{grp}(x)=g\}$ and let $\theta=(\theta_g)_{g\in\mathcal G}\in\Delta^2$ be the latent shortage-priority vector. The per-image loss is
\[
\begin{aligned}
\ell(z,x;\theta):=\sum_{g\in\mathcal G}\Big[
&\theta_g\bigl(y_g(x)-z_g\bigr)_+^2\\[-0.2em]
&+(1-\theta_g)\bigl(z_g-y_g(x)\bigr)_+^2
\Big].
\end{aligned}
\]
The forward problem under image distribution $\mathbb P$ is
\[
z^\star(\mathbb P,\theta)
\in
\arg\min_{z\in\Delta^2}
\mathbb E_{x\sim\mathbb P}\!\left[\ell(z,x;\theta)\right].
\]
For each group, the loss balances a shortage penalty weighted by $\theta_g$ against an overage penalty weighted by $(1-\theta_g)$. If $\pi_g(\mathbb P):=\mathbb P\{\operatorname{grp}(x)=g\}$, taking the expectation yields an equivalent objective in terms of the induced group probabilities, while the stochastic forward problem remains defined over image contexts $x\sim\mathbb P$.

\paragraph{Controlled regimes and shift paths.}
The baseline has clothing-dominant group shares $(0.60,0.25,0.15)$, with representative fine classes \emph{T-shirt/top}, \emph{Sandal}, and \emph{Bag}. The decision-irrelevant $H_0$ endpoint changes the representative fine classes to \emph{Shirt}, \emph{Ankle boot}, and \emph{Bag} while preserving the same group shares, so the image distribution changes but the expected forward objective is unchanged. The balanced harmful $H_1$ endpoint uses the same fine-class support as $H_0$ but assigns probability $1/3$ to each class, yielding group shares $(1/3,1/3,1/3)$ and changing the target-optimal allocation. The orthogonal $H_1$ path starts from $H_0$ and moves along a direction orthogonal to the deployed-risk direction; it changes the optimal allocation while changing the deployed risk only weakly to first order. For each path, $\delta\in[0,1]$ interpolates between its starting mixture and endpoint. The $H_0$ path evaluates Type-I error, while the two $H_1$ paths evaluate power.

\paragraph{Evaluation.}
For each signal level $\delta$, we repeatedly sample fresh baseline and target datasets, estimate the latent preference parameter using the baseline sample and deployed allocation, construct a target-domain challenger on an independent build split, and evaluate the optimality-gap test on a held-out evaluation split. The main text reports partial rejection-rate summaries for the decision-irrelevant path $H_0$ and the orthogonal harmful $H_1$ path across $50$ repetitions. Table~\ref{tab:fashion_mnist_full_results} additionally reports the balanced harmful $H_1$ path at all signal levels.

\begin{table}[!tbh]
\centering
\caption{Full Fashion-MNIST results for the image-only $H_0$ path and two harmful $H_1$ constructions: the balanced group-share shift and the orthogonal objective-blind shift. Entries report the empirical rejection rate $\pm$ the half-width of the 90\% Wilson confidence interval across $50$ independent repetitions at level $\alpha=0.05$.}
\label{tab:fashion_mnist_full_results}
\resizebox{0.98\linewidth}{!}{
\begin{tabular}{llcccc}
\toprule
\textbf{Shift} & \textbf{Magnitude $\delta$} & \textbf{\texttt{RADAR}} & \textbf{X-Mean} & \textbf{X-Distr} & \textbf{Risk-Value} \\
\midrule
$H_0$ & $0.0$ & $0.02 \pm 0.04$ & $0.06 \pm 0.06$ & $0.02 \pm 0.04$ & $0.06 \pm 0.06$ \\
$H_0$ & $0.1$ & $0.02 \pm 0.04$ & $0.06 \pm 0.06$ & $0.14 \pm 0.08$ & $0.02 \pm 0.04$ \\
$H_0$ & $0.2$ & $0.10 \pm 0.07$ & $0.16 \pm 0.08$ & $0.36 \pm 0.11$ & $0.06 \pm 0.06$ \\
$H_0$ & $0.3$ & $0.04 \pm 0.05$ & $0.44 \pm 0.11$ & $0.90 \pm 0.07$ & $0.04 \pm 0.05$ \\
$H_0$ & $0.4$ & $0.06 \pm 0.06$ & $0.64 \pm 0.11$ & $0.96 \pm 0.05$ & $0.04 \pm 0.05$ \\
$H_0$ & $0.5$ & $0.04 \pm 0.05$ & $0.68 \pm 0.11$ & $1.00 \pm 0.03$ & $0.02 \pm 0.04$ \\
$H_0$ & $0.6$ & $0.04 \pm 0.05$ & $0.90 \pm 0.07$ & $1.00 \pm 0.03$ & $0.12 \pm 0.08$ \\
$H_0$ & $0.7$ & $0.02 \pm 0.04$ & $0.96 \pm 0.05$ & $1.00 \pm 0.03$ & $0.08 \pm 0.07$ \\
$H_0$ & $0.8$ & $0.06 \pm 0.06$ & $1.00 \pm 0.03$ & $1.00 \pm 0.03$ & $0.08 \pm 0.07$ \\
$H_0$ & $0.9$ & $0.06 \pm 0.06$ & $1.00 \pm 0.03$ & $1.00 \pm 0.03$ & $0.06 \pm 0.06$ \\
$H_0$ & $1.0$ & $0.06 \pm 0.06$ & $1.00 \pm 0.03$ & $1.00 \pm 0.03$ & $0.04 \pm 0.05$ \\
\midrule
$H_1$ & $0.0$ & $0.02 \pm 0.04$ & $0.06 \pm 0.06$ & $0.02 \pm 0.04$ & $0.06 \pm 0.06$ \\
$H_1$ & $0.1$ & $0.12 \pm 0.08$ & $0.08 \pm 0.07$ & $0.26 \pm 0.10$ & $0.38 \pm 0.11$ \\
$H_1$ & $0.2$ & $0.26 \pm 0.10$ & $0.06 \pm 0.06$ & $0.78 \pm 0.09$ & $0.76 \pm 0.10$ \\
$H_1$ & $0.3$ & $0.50 \pm 0.11$ & $0.30 \pm 0.10$ & $1.00 \pm 0.03$ & $0.94 \pm 0.06$ \\
$H_1$ & $0.4$ & $0.68 \pm 0.11$ & $0.44 \pm 0.11$ & $1.00 \pm 0.03$ & $1.00 \pm 0.03$ \\
$H_1$ & $0.5$ & $0.78 \pm 0.09$ & $0.58 \pm 0.11$ & $1.00 \pm 0.03$ & $1.00 \pm 0.03$ \\
$H_1$ & $0.6$ & $0.96 \pm 0.05$ & $0.72 \pm 0.10$ & $1.00 \pm 0.03$ & $1.00 \pm 0.03$ \\
$H_1$ & $0.7$ & $0.98 \pm 0.04$ & $0.84 \pm 0.08$ & $1.00 \pm 0.03$ & $1.00 \pm 0.03$ \\
$H_1$ & $0.8$ & $0.98 \pm 0.04$ & $1.00 \pm 0.03$ & $1.00 \pm 0.03$ & $1.00 \pm 0.03$ \\
$H_1$ & $0.9$ & $1.00 \pm 0.03$ & $0.98 \pm 0.04$ & $1.00 \pm 0.03$ & $1.00 \pm 0.03$ \\
$H_1$ & $1.0$ & $1.00 \pm 0.03$ & $0.98 \pm 0.04$ & $1.00 \pm 0.03$ & $1.00 \pm 0.03$ \\
\midrule
$H_1$ (orthogonal) & $0.0$ & $0.00 \pm 0.03$ & $0.04 \pm 0.05$ & $0.02 \pm 0.04$ & $0.04 \pm 0.05$ \\
$H_1$ (orthogonal) & $0.1$ & $0.04 \pm 0.05$ & $0.02 \pm 0.04$ & $0.00 \pm 0.03$ & $0.08 \pm 0.07$ \\
$H_1$ (orthogonal) & $0.2$ & $0.12 \pm 0.08$ & $0.08 \pm 0.07$ & $0.02 \pm 0.04$ & $0.04 \pm 0.05$ \\
$H_1$ (orthogonal) & $0.3$ & $0.14 \pm 0.08$ & $0.04 \pm 0.05$ & $0.06 \pm 0.06$ & $0.02 \pm 0.04$ \\
$H_1$ (orthogonal) & $0.4$ & $0.28 \pm 0.10$ & $0.00 \pm 0.03$ & $0.04 \pm 0.05$ & $0.00 \pm 0.03$ \\
$H_1$ (orthogonal) & $0.5$ & $0.34 \pm 0.11$ & $0.00 \pm 0.03$ & $0.26 \pm 0.10$ & $0.02 \pm 0.04$ \\
$H_1$ (orthogonal) & $0.6$ & $0.38 \pm 0.11$ & $0.04 \pm 0.05$ & $0.28 \pm 0.10$ & $0.10 \pm 0.07$ \\
$H_1$ (orthogonal) & $0.7$ & $0.46 \pm 0.11$ & $0.08 \pm 0.07$ & $0.26 \pm 0.10$ & $0.06 \pm 0.06$ \\
$H_1$ (orthogonal) & $0.8$ & $0.56 \pm 0.11$ & $0.16 \pm 0.08$ & $0.48 \pm 0.11$ & $0.02 \pm 0.04$ \\
$H_1$ (orthogonal) & $0.9$ & $0.68 \pm 0.11$ & $0.14 \pm 0.08$ & $0.66 \pm 0.11$ & $0.04 \pm 0.05$ \\
$H_1$ (orthogonal) & $1.0$ & $0.74 \pm 0.10$ & $0.18 \pm 0.09$ & $0.80 \pm 0.09$ & $0.06 \pm 0.06$ \\
\bottomrule
\end{tabular}
}
\end{table}

\subsubsection{Atlanta Police Districting}
\label{app:atlanta}

\paragraph{Dataset and decision variable.}
Police departments deploy patrol resources by partitioning a city into a small number of patrol \emph{zones}, each of which consists of finer-grained \emph{beats}. We study the Atlanta police districting system with $d$ beats and $m$ zones. The incumbent hard districting is represented by $z^0\in[m]^d$, where $z_i^0$ is the zone assigned to beat $i$. Let
\[
n_k:=\bigl|\{i\in[d]:z_i^0=k\}\bigr|,
\qquad k\in[m],
\qquad
\sum_{k=1}^m n_k=d .
\]
We work with a continuous relaxation of the districting decision. Let $P\in\mathbb R^{d\times m}$ denote a soft assignment matrix, where $P_{ik}$ is the weight assigning beat $i$ to zone $k$. The incumbent assignment induces $P^0\in\{0,1\}^{d\times m}$ with
\[
P^0_{ik}=\mathbf{1}\{z_i^0=k\}.
\]
In this application, $P^0$ plays the role of the deployed decision $z_0$ in the main text.

\paragraph{Weekly context signals.}
The incident-level data contain timestamps for call receipt, dispatch, arrival, and clearance. From these timestamps, we construct three nonnegative durations: waiting time, travel time, and on-scene time. For each week $t\in[T]$ and beat $i\in[d]$, define
\[
\begin{aligned}
r^{\mathrm{resp}}_{ti}
&:=
\mathbb E\!\left[
\texttt{waiting\_time\_seconds}
+
\texttt{travel\_time\_seconds}
\mid (t,i)
\right],\\
r^{\mathrm{work}}_{ti}
&:=
\mathbb E\!\left[
\texttt{onscene\_time\_seconds}
+
\texttt{travel\_time\_seconds}
\mid (t,i)
\right],
\end{aligned}
\]
where the expectations are empirical averages over incidents in week $t$ and beat $i$. We stack these quantities into vectors $r^{\mathrm{resp}}_t,r^{\mathrm{work}}_t\in\mathbb R^d$. The weekly context is therefore
\[
X_t:=\bigl(r^{\mathrm{resp}}_t,r^{\mathrm{work}}_t\bigr).
\]

\paragraph{Feasible districting set.}
We impose three types of constraints: each beat must be assigned across zones, incumbent zone sizes are preserved, and assignments must respect an admissibility mask $a_{ik}\in\{0,1\}$. The feasible set is
\[
\mathcal P
:=
\left\{
P\in\mathbb R^{d\times m}:
\sum_{k=1}^m P_{ik}=1\ \forall i,\quad
\sum_{i=1}^d P_{ik}=n_k\ \forall k,\quad
0\le P_{ik}\le a_{ik}\ \forall(i,k)
\right\}.
\]
The mask $a_{ik}$ restricts beat-to-zone assignments to admissible local moves, and the size constraints preserve the number of beats assigned to each zone.

\paragraph{Risk-minimization model.}
We posit that planners trade off response burden and workload burden through a latent scalar preference parameter $\theta\in[0,1]$. For week $t$, define the preference-weighted beat-level signal
\[
r_t^{(\theta)}
:=
\theta r_t^{\mathrm{resp}}
+
(1-\theta)r_t^{\mathrm{work}}.
\]
Given a soft assignment $P$, the induced mean signal in zone $k$ is
\[
\mu_{tk}(P;r_t^{(\theta)})
:=
\frac{1}{n_k}
\sum_{i=1}^d r^{(\theta)}_{ti}P_{ik},
\qquad k\in[m],
\]
and the citywide mean is
\[
\bar\mu_t(r_t^{(\theta)})
:=
\frac{1}{d}\sum_{i=1}^d r^{(\theta)}_{ti}.
\]
Let
\[
c_{tk}(P;r_t^{(\theta)})
:=
\mu_{tk}(P;r_t^{(\theta)})
-
\bar\mu_t(r_t^{(\theta)})
\]
denote the centered zone-level mean. We define the weekly districting loss
\[
f(P;X_t,\theta)
:=
\sum_{k=1}^m
\left[
c_{tk}(P;r_t^{(\theta)})
\right]^2
+
\lambda_{\rm sp}
\sum_{(i,j)\in E}\sum_{k=1}^m
(P_{ik}-P_{jk})^2 ,
\]
where the first term penalizes imbalance across zones and the second term encourages spatial smoothness over the beat adjacency graph $E$, with regularization weight $\lambda_{\rm sp}\ge0$.

We use expectation risk,
\[
R_{\mathbb P}(P,\theta)
=
\mathbb E_{X\sim\mathbb P}\!\left[f(P;X,\theta)\right].
\]
The forward districting problem is therefore
\[
P^\star(\mathbb P,\theta)
\in
\arg\min_{P\in\mathcal P}
R_{\mathbb P}(P,\theta).
\]
The deployment optimality gap of the incumbent districting under distribution $\mathbb P$ is
\[
\Delta(P^0;\theta,\mathbb P)
:=
R_{\mathbb P}(P^0,\theta)
-
\inf_{P\in\mathcal P}R_{\mathbb P}(P,\theta).
\]

\paragraph{Sequential decision-adequacy monitoring.}
We apply the sequential monitoring of Appendix~\ref{subsec:cpd} and Algorithm~\ref{alg:dr_cpd} to the weekly contexts $X_t$, with $P^0$ and $\mathcal P$ in the roles of $z_0$ and $\mathcal Z$. The burn-in is December $2008$ to December $2010$, the incident history against which the incumbent plan was designed; we take the operating environment over this period to be close to that of the months immediately following the plan's adoption in December $2011$, so that monitoring begins from a decision-safe baseline as \eqref{eq:forward_problem} requires. Inverse optimization on the burn-in estimate $\hat\theta_0$, which is held fixed thereafter along with the scale normalization.

Monitoring runs from December $2011$ to March $2021$ on a quarterly grid of $N=38$ times, each using a trailing window of $w=52$ weeks split evenly into a benchmark-construction split and an evaluation split. The benchmark decision $\hat z_t$ solves the empirical districting problem on the benchmark split and the gap of $P^0$ is evaluated on the held-out split, against a relative tolerance $\tau=0.2$ calibrated on the benchmark split, at level $\alpha=0.01$. The decision-agnostic baseline is a two-sample Hotelling test comparing the burn-in contexts with the current window on the same grid and at the same level. Figure~\ref{fig:atlanta_combined}\textbf{(g)} reports both paths.

\end{document}